\documentclass[letterpaper]{article}
\usepackage{aaai25}  
\usepackage{times}  
\usepackage{helvet}  
\usepackage{courier}  
\usepackage[hyphens]{url}  
\usepackage{graphicx} 
\usepackage{natbib}  
\usepackage{caption} 
\usepackage[utf8]{inputenc}

\usepackage{algorithm}
\usepackage[noend]{algpseudocode}
\usepackage{amsthm}
\usepackage{amsmath}
\usepackage{amssymb}
\usepackage{bussproofs}
\usepackage{stackengine}
\usepackage{stmaryrd}
\usepackage{booktabs}
\usepackage[inline]{enumitem}
\usepackage{mathtools}
\usepackage{multirow}
\usepackage{subcaption}
\usepackage{mathtools}
\usepackage{enumitem}
\usepackage{colortbl}
\usepackage{thmtools}

\usepackage[mode=buildnew]{standalone}

\usepackage{pgfplots}
\usepackage{amsfonts}
\usepgfplotslibrary{colorbrewer}

\usetikzlibrary{patterns}

\title{Pseudo-Boolean Proof Logging for Optimal Classical Planning}

\author{
Simon Dold\equalcontrib\textsuperscript{\rm 1},
Malte Helmert\equalcontrib\textsuperscript{\rm 1},
Jakob Nordstr{\"o}m\equalcontrib\textsuperscript{\rm 2, \rm 3},
Gabriele R{\"o}ger\equalcontrib\textsuperscript{\rm 1},
Tanja Schindler\equalcontrib\textsuperscript{\rm 1}
}
\affiliations {
\textsuperscript{\rm 1}University of Basel,\\
\{simon.dold, malte.helmert, gabriele.roeger, tanja.schindler\}@unibas.ch,\\
\textsuperscript{\rm 2}University of Copenhagen,
\textsuperscript{\rm 3}Lund University\\
jn@di.ku.dk
}

\newtheorem{definition}{Definition}
\newtheorem{lemma}{Lemma}
\newtheorem{theorem}{Theorem}

\newcommand{\variables}{\ensuremath{\mathcal{V}}}

\newcommand{\actions}{\ensuremath{\mathcal{A}}}
\newcommand{\pre}{\ensuremath{\textit{pre}}}

\newcommand{\add}{\ensuremath{\textit{add}}}
\newcommand{\del}{\ensuremath{\textit{del}}}
\newcommand{\cost}{\ensuremath{\textit{cost}}}
\newcommand{\init}{\ensuremath{I}}
\newcommand{\goal}{\ensuremath{G}}
\newcommand{\app}[2]{\ensuremath{#1\llbracket #2 \rrbracket}}

\newcommand{\effvars}{\ensuremath{\textit{evars}}}

\newcommand{\astar}{\ensuremath{\textup{A}^*}}
\newcommand{\hmax}{\ensuremath{h^{\textup{max}}}}
\newcommand{\vmax}{\ensuremath{V^{\textup{max}}}}
\newcommand{\boundedvmax}{\ensuremath{W^{\textup{max}}}}

\newcommand{\restrict}[2]{{{#1}\!\!\upharpoonright_{#2}}}
\newcommand{\maxbit}{\lceil\log_2 B\rceil}

\newcommand{\gecost}[1]{\cost_{\geq #1}}
\newcommand{\geprimedcost}[1]{\cost'_{\geq #1}}
\newcommand{\deltacost}[1]{\Delta c^{= #1}}
\newcommand{\eqvar}[1]{eq_{#1,#1'}}
\newcommand{\geqvar}[1]{geq_{#1,#1'}}
\newcommand{\leqvar}[1]{leq_{#1,#1'}}
\newcommand{\cvars}{\mathcal V_c}

  \newcommand{\cinit}{\ensuremath{\mathcal C_\textup{init}}}
  \newcommand{\cgoal}{\ensuremath{\mathcal C_\textup{goal}}}
  \newcommand{\ctrans}{\ensuremath{\mathcal C_\textup{trans}}}
  \newcommand{\cgeq}{\ensuremath{\mathcal C_\geq}}
\newcommand{\ctask}{\ensuremath{\mathcal C_\Pi}}
\newcommand{\etask}{\ensuremath{\mathcal E_\Pi}}
\newcommand{\ccert}{\ensuremath{\mathcal C_\varphi}}
\newcommand{\ccost}{\ensuremath{\mathcal C_K}}

\newcommand{\rinit}{\ensuremath{r_I}}
\newcommand{\rtrans}{\ensuremath{r_T}}
\newcommand{\rgoal}{\ensuremath{r_G}}
\newcommand{\rcert}{\ensuremath{r_{\varphi}}}
\newcommand{\proofinit}{\ensuremath{\mathcal{P}_\textup{init}}}
\newcommand{\proofinductive}{\ensuremath{\mathcal{P}_\textup{ind}}}
\newcommand{\proofgoal}{\ensuremath{\mathcal{P}_\textup{goal}}}
\newcommand{\proofhstate}[1]{\ensuremath{\mathcal{P}_{#1}^h}}
\newcommand{\proofhstateinductive}[1]{\ensuremath{\mathcal{P}_{#1,\textup{ind}}^h}}
\newcommand{\proofhstategoal}[1]{\ensuremath{\mathcal{P}_{#1,\textup{goal}}^h}}
\newcommand{\Hpdb}{H_{\textup{PDB}}}

\newcommand{\rmax}{r^{\textup{max}}}
\newcommand{\Hmax}{H_{\textup{max},s}}

\newcommand{\rstateming}[2]{r_{#1,\textup{g}\geq #2}}
\newcommand{\reprset}{M}
\newcommand{\open}{\textit{Open}}
\newcommand{\closed}{\textit{Closed}}

\newcommand{\mmgecost}[1]{K_{\geq #1}}
\newcommand{\mmgeprimedcost}[1]{K'_{\geq #1}}
\newcommand{\mmvars}{\mathcal K}

\newcommand{\rpdbstate}[1]{r^{#1}}
\newcommand{\rpdbprimedstate}[1]{{r^{#1}}'}
\newcommand{\rpdbstategeg}[1]{r^{#1}_{\geq B-d(#1)}}
\newcommand{\rpdbprimedstategeg}[1]{{r^{#1}_{\geq B-d(#1)}}'}
\newcommand{\pdbstate}{s_{\alpha}}
\newcommand{\aalpha}{a^{\alpha}}
\newcommand{\rpdbinfinite}{r^{\infty}}

\newcommand{\inlinecite}[1]{\citeauthor{#1} \shortcite{#1}}
\newcommand{\egcite}[1]{\citep[e.g.,][]{#1}}

\newcommand{\includefigure}[3][scale=1]{
  \begin{figure}[t]
    \centering
    \noindent
    \includegraphics[#1]{figures-#2}
    \caption{#3}
    \label{figure:#2}
  \end{figure}
}

\newcommand{\dq}[1]{``#1''}

\newboolean{icaps} 
\setboolean{icaps}{false}

\ifthenelse{\boolean{icaps}}
    {}
    {\nocopyright}

\begin{document}

\maketitle

\begin{abstract}
  {
    We introduce lower-bound certificates for classical planning
    tasks, which can be used to prove the unsolvability of a task or
    the optimality of a plan in a way that can be verified by an
    independent third party. We describe a general framework for
    generating lower-bound certificates based on pseudo-Boolean constraints, which
    is agnostic to the planning algorithm used.

    As a case study, we show how to modify the $\astar$ algorithm to
    produce proofs of optimality with modest overhead,
    using pattern database heuristics and $\hmax$ as concrete examples.
    The same proof logging approach works for any heuristic whose inferences
    can be efficiently expressed as reasoning over pseudo-Boolean constraints.
  }
\end{abstract}

\section{Introduction}

Optimal classical planning algorithms make three promises: that the
plans they produce are correct, that no cheaper plans achieving the
goal exist, and that any task reported as unsolvable actually is. As
\inlinecite{mcconnell-et-al-csr2011} argue in their seminal work on
\emph{certifying algorithms}, there are many good reasons not to
accept such promises blindly. Instead, a \emph{certifying planning
algorithm} outputs some kind of proof (a \emph{certificate}) that an
independent third party can use to verify the truthfulness of the
planner's claims.

For the correctness of plans, such a verification is performed
routinely: the generated plans themselves serve as certificates for
this, and plan validation tools such as VAL
\cite{howey-long-icaps2003wscompetition} or INVAL
\cite{haslum-github2016-accessed-2025-05-02} can be used to check that the produced plans
are correct. For unsolvability, Eriksson et
al.~\cite{eriksson-et-al-icaps2017,eriksson-et-al-icaps2018,eriksson-helmert-icaps2020}
recently introduced two forms of unsolvability certificates that cover
very diverse planning algorithms.

Certifying the \emph{optimality} of plans, in contrast, is still in
its infancy. The only existing work in this direction is a paper by
\inlinecite{mugdan-et-al-icaps2023} which describes two approaches:
one based on a compilation to unsatisfiability, and one based on an
extension of the unsatisfiability proof system of
\inlinecite{eriksson-et-al-icaps2018}. The first approach is in
general not computationally feasible, as it requires a task
reformulation that increases the number of state variables and actions
exponentially in the size of the planning task. The second approach
does not share this weakness, but is not sufficiently general to
encompass a wide range of planning approaches. \citeauthor{mugdan-et-al-icaps2023} identify
five essential properties for optimality certificates:
\emph{soundness} (if an optimality certificate is accepted by the
verifier, then optimality holds), \emph{completeness} (if a solution
is optimal, then a certificate of its optimality exists),
\emph{efficient generation} (a non-certifying algorithm can be changed
to produce certificates with reasonable, at most polynomial overhead),
\emph{efficient verification} (the verifier runtime is
polynomial in the task and certificate size), and
\emph{generality} (certificates can be efficiently produced by a wide
variety of different planning algorithms rather than being
algorithm-specific).

\citeauthor{mugdan-et-al-icaps2023} show that their approach is sound and complete, but
critically comment that ``the other three properties for practical
usability do not have clear-cut answers''. Compared to the
unsolvability proof system of Eriksson et al., their approach appears
much more specific to heuristic forward search using heuristics that
are themselves based on some kind of forward search, such as the
\hmax\ heuristic \cite{bonet-geffner-aij2001}. For example, it is not
at all clear how to certify heuristics computed in a backward
direction, such as pattern databases \cite{edelkamp-ecp2001} or other
abstraction heuristics without redoing the abstract state space
exploration for every heuristic evaluation.

While in the planning community the interest in certifying algorithms
is quite recent, they are standard in the SAT community, where
unsatisfiability certificates based on proof systems like DRAT
\cite{heule-et-al-cade2013} and LRAT \cite{cruz-filipe-et-al-cade2017}
are required for participating in SAT competitions and supported by
formally verified checkers \cite{tan-et-al-sttt2023,lammich-jar2020}.
A recently proposed alternative are pseudo-Boolean proofs based
on cutting planes, which are supported by the formally verified
checker VeriPB \cite{bogaerts-et-al-aaai2022}. Unlike the other
proof systems mentioned, cutting planes are able to
directly incorporate linear arithmetic, making them very appealing for
optimization problems such as optimal classical planning. Indeed,
VeriPB has been applied to a wide range
of optimization problems such as
MaxSAT \cite{berg-et-al-cade2023},
ILP presolving \cite{hoen-et-al-cpaior2024},
constraint programming
\cite{gocht-et-al-cp2022,mcilree-mccreesh-cp2023,mcilree-et-al-cpaior2024},
and dynamic programming \cite{demirovic-et-al-cp2024}.

\includefigure[width=.8\linewidth]{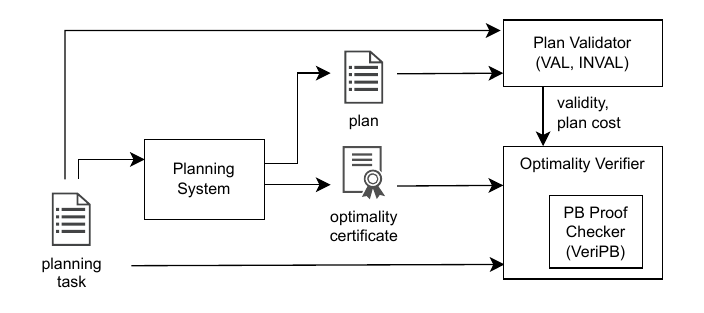}{Interaction of Components}

In this work, we propose to certify plan optimality by means of such
pseudo-Boolean proofs. We claim that our optimality certificates have all five
desirable properties discussed above: they are sound, complete,
efficiently generatable, efficiently verifiable, and general. (For
space reasons, the generality of the proof system must remain a
conjecture for the time being, but at least we show that we can cover
the techniques covered by Mugdan et al.\ as well as pattern database
heuristics that their approach struggles with.)

The overall concept is shown in Figure \ref{figure:optimalityproofs}.
In addition to a plan, the certifying planning system produces an
optimality certificate that proves that the input task has no cheaper
solution. A plan validator such as VAL verifies that the
plan solves the task and determines its cost. The optimality
verifier has access to the original task, the plan cost determined by
the validator, and the certificate from the planning system. On this
basis it verifies that the plan is indeed optimal.

The certificate is based on cutting planes proofs with reification
\cite{bogaerts-et-al-jair2023}, in which the atomic pieces of
knowledge are pseudo-Boolean constraints. Such constraints allow us to
reason conveniently about costs of actions and bounds on costs to
reach a state. Roughly speaking, the certificate describes an
overapproximation of which states can be reached at which cost and
shows that no goal state can be reached at a cost below the claimed
optimal plan cost. The heart of the proof is a standard pseudo-Boolean
proof that can be verified by an off-the-shelf proof checker such as
the VeriPB system.

The paper is structured as follows: we introduce the general framework
for certifying plan optimality by providing an encoding of planning
tasks as pseudo-Boolean constraints and defining lower-bound
certificates. Next, we define heuristic certificates, which represent
the concept of an admissible heuristic estimate as pseudo-Boolean
constraints. We then show how heuristic certificates can be used to
integrate proof-logging into \astar, capturing the arguments for
optimality of \astar\ with an admissible heuristic on a given state
space as pseudo-Boolean constraints. Finally, we describe how
heuristic certificates for pattern database heuristics and $\hmax$ can
be generated, which amounts to proving the admissibility of these
heuristics with pseudo-Boolean constraints.

\section{Background}
We first introduce the STRIPS planning formalism
\cite{fikes-nilsson-aij1971}, followed by pseudo-Boolean constraints and the
cutting planes with reification proof system.

\subsection{STRIPS Planning Tasks}

A STRIPS planning task $\Pi = \langle \variables, \actions, \init,
\goal\rangle$ consists of a finite set $\variables$ of propositional
\emph{state variables}, a finite set $\actions$ of \emph{actions}, the
\emph{initial state} $\init\subseteq \variables$ and the \emph{goal}
$\goal\subseteq \variables$.

A \emph{state} $s\subseteq \variables$ of $\Pi$ induces a variable assignment
$\rho_s$ that maps every $v\in s$ to $1$ and every $v\notin s$ to
$0$.
Each action $a\in \actions$ is a tuple $\langle\pre(a),\add(a),\del(a),\cost(a)\rangle$,
where $\pre(a)\subseteq\variables$ is the set of \emph{preconditions},
$\add(a)\subseteq\variables$ is the set of \emph{add effects},
$\del(a)\subseteq \variables\setminus\add(a)$ is the set of \emph{delete
effects}, and $\cost(a)\in \mathbb{N}_0$ is the \emph{cost} of $a$.
We write $\effvars(a)$ for the set $\add(a)\cup\del(a)$ of affected variables.
Action $a$ is \emph{applicable} in state $s$ if $\pre(a)\subseteq s$.
The \emph{successor state} is $\app{s}{a} = (s\setminus\del(a))\cup\add(a)$.
For a sequence $\pi=a_1,\dots,a_n$ of actions that are successively applicable
in state $s$, we write $\app{s}{\pi}$ for the resulting state $\app{s}{a_1}\dots\llbracket a_n
\rrbracket$. If this state is a \emph{goal state}, i.e.\ $\goal\subseteq
\app{s}{\pi}$, then $\pi$ is a \emph{plan} for $s$. A plan for the initial
state $\init$ is a plan for task $\Pi$. The cost of $\pi$ is
$\cost(\pi) = \sum_{i=1}^n\cost(a_i)$. A plan is \emph{optimal} if there is no
plan of lower cost. Task $\Pi$ is solvable if there is a plan for $\Pi$,
otherwise it is \emph{unsolvable}.

\subsection{Pseudo-Boolean Formulas}
A \emph{Boolean} variable has domain $\{0,1\}$. A literal $\ell$ is a Boolean variable $x$ or
its negation $\bar x$. The negation of a literal $\ell = \bar x$ is $\bar \ell = x$.
A \emph{pseudo-Boolean (PB) constraint} (in normalized form) over a finite set $X=\{x_1,\dots,x_n\}$ of
Boolean variables is an inequality
\begin{equation*}
  \sum\nolimits_{i} a_i \ell_i \geq A,
\end{equation*}
where all literals $\ell_i$ are over distinct variables from $X$, $A\in\mathbb N_0$ and all
coefficients $a_i$ are from $\mathbb N_0$.
We will also write linear constraints more flexibly, but they can always be
transformed to normalized form by simple algebraic transformations. As
syntactic sugar, we will also write $(\ell_1\land\dots\land\ell_n)\rightarrow \ell$
as abbreviation for $\overline{\ell_1} + ... + \overline{\ell_n} + \ell \geq 1$.

A \emph{solution} of the constraint is an assignment
$\rho:X\rightarrow\{0,1\}$ such that the inequality is satisfied by replacing
every variable $x_i$ with $\rho(x_i)$ and every negated variable $\bar x_i$
with $1-\rho(x_i)$. We also use assignments for all literals, implicitly
requiring that $\rho(\bar x ) = 1-\rho(x)$ for all variables $x$.

For constraint $C \doteq \sum\nolimits_{i} a_i \ell_i \geq
A$, the negation $\lnot C$ is the normalized form of $\sum\nolimits_{i}
a_i \ell_i \leq A - 1$.

A \emph{pseudo-Boolean (PB) formula} \cite{buss-nordstroem-hos2021} is a finite set
$\mathcal C$ of pseudo-Boolean constraints over a set $X$ of
variables. An assignment $\rho:X\rightarrow\{0,1\}$ is a \emph{model} of the formula if it is
a solution of all constraints. If $\mathcal C$ has no model, it is
\emph{unsatisfiable}. We say that a constraint $C$ is \emph{implied} by a PB
formula $\mathcal C$ (written $\mathcal C\models C$) if every model of
$\mathcal C$ is a solution of $C$ and that PB formula $\mathcal D$ is implied
by $\mathcal C$ (written $\mathcal C\models\mathcal D$) if $\mathcal C\models
D$ for all $D\in \mathcal D$.

For a partial variable assignment $\rho:X\nrightarrow\{0,1\}$ and constraint
$C$, we write $\restrict{C}{\rho}$ for the constraint obtained from $C$
by replacing each variable $v$ in the domain of $\rho$ by $\rho(v)$ and
normalizing.

\subsection{Cutting Planes with Reification Proof System}

The VeriPB proof system \cite{bogaerts-et-al-jair2023} is an extension of the
cutting planes proof system. We use a subset of the VeriPB proof
system which we call
\emph{cutting planes with reification} (CPR).

Proofs in the cutting planes proof system \cite{cook-et-al-dam1987}
are built from one axiom and three derivation rules. For any literal
$\ell$, the \emph{literal axiom} allows us to derive $\ell\geq 0$
without prerequisites. If we already have $\sum\nolimits_{i} a_i
\ell_i \geq A$ and $\sum\nolimits_{i} b_i \ell_i \geq B$, we can
derive
\begin{itemize}
\item $\sum\nolimits_{i} (c_Aa_i + c_Bb_i) \ell_i \geq c_A A + c_B B$\\for every $c_A, c_B\in\mathbb N_0$
(\emph{linear combination}),
\item $\sum\nolimits_{i} \lceil a_i/c\rceil \ell_i \geq \lceil A/c\rceil$\\ for
  every $c\in\mathbb N^+$ (\emph{division}), and
\item $\sum\nolimits_{i} \min\{a_i,A\} \ell_i \geq A$ (\emph{saturation}).
\end{itemize}

All constraints that can be derived from a PB formula $\mathcal
C$ by these rules are implied by $\mathcal C$.

\subsubsection{Reverse Unit Propagation} A constraint $C\in \mathcal C$ \emph{unit propagates} literal $\ell$ under
partial assignment $\rho$ if all models of $\restrict{C}{\rho}$ assign $\ell$ to
$1$.
When this happens, we can extend $\rho$ with $\ell\mapsto 1$ and unit-propagate
further literals under the extended assignment.
If this process derives a conflict (assigning 0 and 1 to the same
variable)
starting from the empty variable assignment, then $\mathcal C$ is
unsatisfiable.

Formula $\mathcal C$ implies constraint $C$ by \emph{reverse unit
propagation} (RUP) if $\mathcal C\cup\{\lnot C\}$ unit propagates to a
conflict from the empty assignment. Any constraint implied by RUP can
be derived by a cutting plane proof. As in VeriPB, our proof system
allows adding any constraint implied by RUP directly in a single step.
This is a useful shortcut that drastically compresses many cutting
plane proofs.

\subsubsection{Reification}

In addition to cutting planes reasoning, our proof system also
allows \emph{reification}, i.e., introducing a new variable that
represents the truth value of a constraint. If $C$ is a constraint and
$r$ is a new variable, we
write $r \Leftrightarrow C$ to express that variable $r$ must be 1
if $C$ is true under the assignment and 0 otherwise.
If $C$ is $\sum\nolimits_{i} a_i \ell_i \geq A$, this is a shorthand notation
for the two constraints
\begin{align*}
  A\bar r + \sum\nolimits_{i} a_i \ell_i &\geq
  A,\text{ and}
  \\
  (M - A + 1)r + \sum\nolimits_{i} a_i \bar \ell_i &\geq
  M - A + 1
\end{align*}
where $M = \sum\nolimits_{i}a_i$. We use the notation $r \Rightarrow
C$ for the first and $r \Leftarrow C$ for the second constraint.

To summarize, a \emph{CPR proof} consists of a sequence of derivation
steps from the cutting plane proof system (literal axiom, linear
combination, division, and saturation), constraints derived by RUP,
and reifications. In addition, we allow the \emph{redundance-based
strengthening} (RED) rule from VeriPB, which can be understood as a
form of proof by contradiction. We only use it in
\ifthenelse{\boolean{icaps}}{an extended version of this paper
  \cite{dold-et-al-arxiv2025}}{the appendix} and therefore describe it
there.

\section{Lower-Bound Certificates}

We propose certificates that prove that there is no plan of lower cost
than a given bound $B$. A typical, but not the only, application of
such lower-bound certificates is to certify that a plan of cost $B$ is
optimal.
Before we define the full framework, we first introduce
how we encode planning tasks by means of PB formulas.

\subsection{Encoding Planning Tasks}

The PB encoding of task $\Pi = \langle \variables, \actions, \init,
\goal\rangle$ uses the propositional variables from $\variables$ as
Boolean variables, and variables $\cvars
= \{c_0,\dots,c_{\lceil\log_2(B)\rceil}\}$ as a binary representation
of a number in the range $0,\dots,B$. These variables allow us to
represent pairs $\langle s, c\rangle$ consisting of a state $s$ and a
number $c \leq B$. In the following, one can use the intuition that $s$
is a state of the task, which has been reached incurring cost $c$.
In addition, we introduce a number of reification variables.

Reification variable $\rinit$ is true in a model iff the
state variables encode the initial state:
\begin{equation}
\rinit\Leftrightarrow \sum_{v\in\init}
  v + \sum_{v\in\variables\setminus\init}\bar v \geq |\variables|
\label{eq:rinit}
\end{equation}

For the goal, we introduce a reification variable $\rgoal$, which is true in
a model iff
the state variables encode a goal state:
\begin{equation}
  \rgoal\Leftrightarrow \sum_{v\in\goal} v \geq |\goal|
\label{eq:rgoal}
\end{equation}

For the actions, we encode transitions from a state $s$ to successor
state $\app{s}{a}$ in a similar way to symbolic search
\egcite{edelkamp-kissmann-ijcai2009} or planning as satisfiability
\egcite{rintanen-et-al-aij2006},
encoding the successor state by means of additional variables $v'$ for each
state variable $v$. The variables $c_i$ encode a cost by which state
$s$ can be reached, and analogously we use a variable $c'_i$ for
each $c_i$ to encode the cost to reach $\app{s}{a}$ via this transition. For
this purpose, we need constraints that ensure that the difference between the
two values corresponds to the cost of the action. We do this by means of
additional reification variables $\deltacost{k}$ that express that the
difference between the two numbers is $k$:
\begin{equation}
  \deltacost{k}\Leftrightarrow \sum\nolimits_{i=0}^{\maxbit} 2^ic'_i
  - \sum\nolimits_{i=0}^{\maxbit} 2^ic_i = k
  \label{eq:deltacost}
\end{equation}

To express that the variables $c_i$ or $c'_i$ encode a value that is
at least $k$ for some $k\in\{0,\dots,B\}$, we use reification
variables $\gecost{k}$ and $\geprimedcost{k}$:

\begin{align}
  \gecost{k} &\Leftrightarrow \sum\nolimits_{i=0}^{\maxbit} 2^ic_i \geq k
  \label{eq:gecost}\\
  \geprimedcost{k} &\Leftrightarrow \sum\nolimits_{i=0}^{\maxbit} 2^ic'_i \geq k
  \label{eq:geprimedcost}
\end{align}

Note that we do not introduce these reification variables for
all values of $k$ up to $B$, which would require an exponential number
of variables in the encoding size of $B$. Rather, we lazily introduce
only the variables used by the proof.

For handling the state variables $v$ that are not affected by an action, we
introduce reification variables $\eqvar{v}$ that are true in a model iff it
assigns $v$ and $v'$ the same value:
\begin{equation}
  \begin{aligned}
    \eqvar{v} &\Leftrightarrow \leqvar{v} + \geqvar{v} \geq 2\\
    \geqvar{v} &\Leftrightarrow v + \overline{v'} \geq 1\\
    \leqvar{v} &\Leftrightarrow \bar{v} + v' \geq 1
  \end{aligned}
  \label{eq:eqvars}
\end{equation}

For each action $a\in\actions$ we introduce a variable $r_a$ expressing
that whenever the action is applied, the cost is increased
by $\cost(a)$, the action precondition is satisfied, the primed
variables truthfully represent the successor state, and the successor
cost is within the cost bound:
\begin{equation}
\begin{aligned}
  r_a\Rightarrow {}&\deltacost{\cost(a)} + \sum_{\mathclap{v\in\pre(a)}}
  v + \sum_{\mathclap{v\in\add(a)}}v'
  + \sum_{\mathclap{v\in\del(a)}}\overline{v'} \\
  &{} + \sum_{\mathclap{v\in \variables\setminus\effvars(a)}}  \eqvar{v}
   + \overline{\geprimedcost{B}}\geq 2 + |\pre(a)| + |\variables|
\end{aligned}
  \label{eq:actionvar}
\end{equation}

Observe that this constraint represents a conjunction: all literals must
be true to meet the bound $2 + |\pre(a)| + |\variables|$.

Finally, reification variable $\rtrans$ encodes that a state
transition happens, i.e., some action variable is
selected:
\begin{equation}
\rtrans\Leftrightarrow \sum_{a\in\actions} r_a \geq 1
  \label{eq:rtrans}
\end{equation}
This representation allows selecting several actions
at the same time, but only if they all lead to the same state change
under the same cost.

\begin{definition}
  For planning task $\Pi = \langle \variables, \actions, \init, \goal\rangle$
  and cost bound $B\in\mathbb N_0$, a \emph{PB task encoding} is a tuple
  $\etask = \langle\ctask, \rinit, \rgoal, \rtrans\rangle$, where
  $\ctask=\langle\cinit,\cgoal,\ctrans,\cgeq\rangle$ such that
  $\cinit$, $\cgoal$, $\ctrans$ and $\cgeq$ are sets of reifications from equations \eqref{eq:rinit}--\eqref{eq:rtrans},
  $\cinit$ from \eqref{eq:rinit} and \eqref{eq:gecost},
  $\cgoal$ from \eqref{eq:rgoal} and \eqref{eq:gecost},
  $\ctrans$ from \eqref{eq:deltacost}--\eqref{eq:rtrans},
  $\cgeq$ from \eqref{eq:deltacost}--\eqref{eq:geprimedcost}
  and $\rinit$, $\rgoal$, and $\rtrans$ are the
  reification variables introduced in \eqref{eq:rinit}, \eqref{eq:rgoal}, and
  \eqref{eq:rtrans}.
\end{definition}

\subsection{Certifying Unsolvability under Cost Bound}

A certificate shows that the task is unsolvable under a cost bound $B$, i.e.,
there is no plan $\pi$ with $\cost(\pi) < B$.

Intuitively, it is based on an invariant $\varphi$ which represents an overapproximation of the reachable state-cost pairs.
An invariant in general is a property that is preserved through action
applications, i.e. (1) whenever $\varphi$ is true for state $s$
and cost $c$, and action $a$ is applicable in $s$, then $\varphi$ is true for
state $\app{s}{a}$ and cost $c + \cost(a)$.\footnote{For unsolvability under
a cost bound, it is sufficient to only require the invariant property for
action applications within the cost bound. We implicitly do this because our
task encoding only permits such action applications, cf.\ Eq.\
\eqref{eq:actionvar}.}
In addition, we require that (2)
$\varphi$ is true for the initial state and cost $0$. Together, (1) and (2)
ensure that $\varphi$ is true in all reachable state-cost pairs. If in
addition, (3) $\varphi$ is \emph{not} true for any goal state with a cost
strictly lower than $B$, this implies that the task cannot be solved with cost
$< B$: initially, the invariant is true (2), it is impossible to change this by
an action application (1), so for all reachable goal states the incurred cost
is at least $B$ (3).

In our certificates, the invariant will be defined by a sequence of PB reifications,
and the certificate must prove the three properties above by means of three
separate CPR proofs. We call property (1)  the \emph{inductivity lemma},
property (2) the \emph{initial state lemma} and property (3) the \emph{goal
lemma}.

We can think of the representation of the invariant as a circuit where
each gate evaluates a PB constraint. The state variables and cost bits
$c_i$ are the inputs to the circuit, and the output of the circuit
determines whether the invariant is true for the input state-cost pair. We now formalize this notion.

\begin{definition}
  A \emph{PB circuit} with input variables $V$ is a pair $\mathcal R =
  \langle R, r\rangle$,
  where 
  \begin{itemize}
    \item $R$ is a sequence $\langle r_1\Leftrightarrow \varphi_1, \dots,
      r_n\Leftrightarrow \varphi_n\rangle$ of PB reifications such that each PB
      constraint $\varphi_i$ only has non-zero coefficients for variables from
      $V\cup\{r_j\mid j<i\}$, and 
    \item $r \in \{r_1,\dots,r_n\}$ is the \emph{output variable}.
  \end{itemize}
\end{definition}
Slightly abusing notation, we also interpret a sequence of PB constraints
as the PB formula consisting of its components (i.e., treat the sequence
as a set).
For a PB circuit $\langle R,r\rangle$ with input variables
$\variables\cup\cvars$, we will write $\reprset(R, r)$ for the
represented set of pairs $\langle s,c\rangle$.
In other words, $\reprset(R, r)$ contains the pair $\langle s,c\rangle$ iff
there exists a model $\rho$ of $R\cup\{r = 1\}$, such that $s
= \{v\in\variables\mid \rho(v)=1\}$ and $c = \sum_i 2^i \rho(c_i)$.

Putting the pieces together, we get the following formal definition of
lower-bound certificates. Remember that in the encoding of the planning task,
$\variables$ contains the state variables and $\cvars$ the binary variables for
representing cost.

\begin{definition}
  Let $\Pi = \langle \variables, \actions, \init, \goal\rangle$ be a planning
  task and $B\in\mathbb N_0$ a cost bound. Let $\mathcal E_\Pi = \langle\ctask, \rinit, \rgoal,
  \rtrans\rangle$ be a PB task encoding for $\Pi$ and $B$.

  A \emph{lower-bound certificate for $\Pi$ with bound $B$} is a tuple
  $\langle\langle\ccert, \rcert\rangle, \proofinit, \proofinductive, \proofgoal\rangle$, where
  \begin{itemize}
    \item $\langle\ccert, \rcert\rangle$ is a PB circuit with input variables $\variables
      \cup\cvars$ not mentioning a primed variable.
    \item \textbf{initial state lemma:} $\proofinit$ is a CPR proof for\\
      $\cinit\cup\ccert\cup\cgeq \models (\rinit \land
       \overline{\gecost{1}})\rightarrow \rcert$.
    \item \textbf{goal lemma:} $\proofgoal$ is a CPR proof for\\
      $\cgoal\cup\ccert\cup\cgeq \models (\rgoal\land {\rcert})\rightarrow {\gecost{B}}$.
    \item \textbf{inductivity lemma:} $\proofinductive$ is a CPR proof for\\
      $\ctrans\cup\ccert\cup\ccert'\cup\cgeq \models (\rcert\land
       \rtrans)\rightarrow\rcert'$,
      where $\ccert'$ is a copy of $\ccert$ where all PB variables are replaced
      with their primed version.
  \end{itemize}
  \label{def:certificate}
\end{definition}

Lower-bound certificates are sound:

\begin{theorem}
If there is a lower-bound certificate for planning task $\Pi$ with bound $B$,
then the task has no plan $\pi$ with $\cost(\pi) < B$.
\end{theorem}

\begin{proof}
  Assume there is a plan $\pi = \langle a_1,\dots,a_n\rangle$ with $\cost(\pi) < B$.
  It induces a sequence of state-cost pairs $\langle s_0,
  g_0\rangle,\dots,\langle s_n, g_n\rangle$ such that $s_0$ is the initial
  state, $g_0 =0$, and for all $i\in\{1,\dots,n\}$, $s_i = \app{s_{i-1}}{a_i}$
  and $g_i = g_{i-1} + \cost(a_i)$. Note that  $s_n$ is a goal state and
  $g_n = \cost(\pi) < B$.

  By the initial state lemma, $\langle s_0, g_0\rangle$ is in $\reprset(\ccert,
  \rcert)$. In $\rtrans$, the PB encoding captures all action applications that do not
  exceed the cost bound, which is the case for the action applications in $\pi$
  because $\cost(\pi) < B$. Thus, by the inductivity lemma, we get that
  $\langle s_i, g_i\rangle\in\reprset(\ccert, \rcert)$
  for all $i\in\{1,\dots,n\}$. By the goal lemma, we have for all $\langle
  s, g\rangle\in\reprset(\ccert, \rcert)$ that $g\geq B$ if $s$ is a goal
  state. Since $s_n$ is a goal state, and $\langle s_n,g_n\rangle\in\reprset(\ccert,
  \rcert)$, this implies that $g_n\geq B$, a contradiction to $g_n < B$.
\end{proof}

The verifier component (see Fig.~\ref{figure:optimalityproofs})
receives the planning task, the lower-bound certificate, and the
validated plan cost. Inside, it confirms that the reifications in
$\mathcal{E}_\Pi$ are a system of pseudo-Boolean constraints that
encode the original planning task. Additionally, it generates
$\langle\mathcal{C}'_\varphi,r'_\varphi\rangle$ based on
$\langle\mathcal{C}_\varphi,r_\varphi\rangle$ by creating a copy of
each PB reification in $\mathcal{C}_\varphi$ where all variables are
replaced by their primed counterparts. It then verifies that the
proofs $\proofinit$, $\proofgoal$ and $\proofinductive$ actually
derive the statements in Definition~\ref{def:certificate}. Only if all
these steps are successful, the verifier confirms that the provided
plan cost is optimal.

This concludes our formalization of lower-bound certificates and their
verification. Lower-bound certificates are efficiently verifiable: the
verifier runs in polynomial time because the underlying
PB proof checker does and all steps other than the proof checking are
easy to perform in polynomial time.
In the rest of the paper we describe a case study of how certain
optimal classical planning algorithms can be augmented to efficiently
generate lower-bound certificates. Because the algorithms we discuss
are complete, this also shows that lower-bound certificates are
complete.

\section{Proof-Logging Heuristic Search}

We now examine how a planner can be transformed into a proof-logging
system, in our case a planner that emits a lower-bound certificate for
the cost of the found plan. The idea is to already log most relevant
information about its internal reasoning during its normal operation,
keeping the overhead for generating the proof low.

$\astar$ search \cite{hart-et-al-ieeessc1968} with an admissible heuristic $h$, i.e.\
a distance estimator that provides a lower bound on the goal distance of
a state, is the most common approach for optimal planning as heuristic
search.

It maintains a priority queue $\open$ of states $s$, ordered by
$f=g+h(s)$, where $g$ is the best known cost to reach $s$ from the initial
state. The open list is initialized with the initial state. $\astar$ removes
a state $s$ from the list and expands it until it removes a goal state. A state
expansion generates all successor states and adds them to the open list if the
implicit path via $s$ improves its $g$-value.

One challenge when transforming the approach into a proof-logging system is
that the bound $B$ that should be certified is only known once a plan has been
found. To handle this, we will use ``placeholder'' PB variables
$\mmgecost{l}$ and $\mmgeprimedcost{l}$ with integers $l$. Once $B$ is known, the
proof will be augmented by reifications that give them the intended meaning.
The idea is that they represent $\gecost{l}$ and $\geprimedcost{l}$,
but $l$ will not necessarily be in the range $\{0,\dots,B\}$. For
$l<0$, we will instead use $0$, and for $l>B$, we will use $B$. So the
reifications for these variables will be
\begin{align}
  \mmgecost{l} &\Leftrightarrow \gecost{\min\{B,\max\{0, l\}\}}\text{
    and}\label{eq:mmgecost}\\
  \mmgeprimedcost{l} &\Leftrightarrow \geprimedcost{\min\{B,\max\{0,
  l\}\}}\label{eq:mmgeprimedcost}
\end{align}
By $\mmvars$, we denote the (infinite) set of all such variables. Upon
termination $\astar$ adds a set $\ccost$ of reifications for
the finitely many such variables that actually occur in the proof. In the
proofs in this paper, we will silently switch from $\mmgecost{\cdots}$ to the
corresponding $\gecost{\cdots}$ expression, implicitly using
\eqref{eq:mmgecost}.

The following two lemmas establish that these clipped costs behave as
expected.\footnote{Throughout the paper, we omit proofs of technical
lemmas that do not provide further insight. These proofs are
included in the
\ifthenelse{\boolean{icaps}}{extended version of this paper
  \cite{dold-et-al-arxiv2025}}{appendix}.}
Intuitively, the first lemma says that if the
costs represented by the variables in $\cvars$ exceed $j+k$, they also exceed the
smaller value $j$:

\begin{restatable}{lemma}{costsone}
  Let $j\in\mathbb Z$ and $k\in\mathbb N_0$. It is possible to derive
  $\gecost{\min\{B,\max\{0,j+k\}\}}\rightarrow
  \gecost{\min\{B,\max\{0,j\}\}}$ from $\cgeq$.
  \label{lemma:costs1-variant}
\end{restatable}

The second lemma informally states that if the costs already exceed $l$ and
we spend cost $m$, the successor cost exceeds $l+m$ (clipping all costs into
$\{0,\dots,B\}$).
\begin{restatable}{lemma}{coststwo}
  For $l\in\mathbb Z$ and $m\in\mathbb N_0$ it is possible to derive
  $(\gecost{\min\{B,\max\{0, l\}\}}\land \deltacost{m})\rightarrow\geprimedcost{\min\{B,\max\{0,l+m\}\}}$ from $\cgeq$.
  \label{lemma:costs3}
\end{restatable}

\subsection{Proof-Logging Heuristics}

Parts of the overall proof must be contributed by the heuristic. Intuitively,
this contribution can be seen as a certificate that shows that it is impossible
to reach the goal with a total cost strictly less than $B$ from a state $s$
that has been reached incurring a cost of at least $c$. For an admissible heuristic, this
will naturally be the case for all $c$ where $c + h(s) \geq B$. We will later
see examples for different heuristics where some use the same invariant for
all states and others generate a new invariant for each evaluated state.

Throughout its execution (a potential initialization phase and evaluations
for a number of states), the heuristic maintains its own PB circuit
$\langle H,r^h\rangle$ with input variables $\variables$ and $\mmvars$. We assume
that the heuristic uses its own
namespace, so it does not introduce reification variables that are also
introduced by the search or some other heuristic. Whenever the search uses the
heuristic to evaluate a state $s$, the heuristic not only returns the
estimate but also a PB variable $r^h_s$ with the following
requirements:

\begin{itemize}
  \item $\reprset(H, r^h_s)$ contains all pairs $\langle s, \hat c\rangle$ with $\hat
    c + h(s) \geq B$.
  \item If $\reprset(H, r^h_s)$ contains pair $\langle \hat s, \hat c\rangle$, it
    contains all pairs $\langle\tilde s, \tilde c\rangle$ such that $\tilde s$ is
    reachable from $\hat s$ with overall cost $\tilde c < B$ (i.e., there is an
    action sequence $\pi$ with $\app{\hat s}{\pi} = \tilde s$ and
    $\tilde c = \hat c + \cost(\pi) < B$).
  \item $\reprset(H, r^h_s)$ contains no pair $\langle\hat s, \hat c\rangle$ where $\hat s$ is a goal state and $\hat c < B$.
\end{itemize}

This motivation leads to the following definition:

\begin{definition}
  Let $\Pi = \langle \variables, \actions, \init, \goal\rangle$ be a planning
  task and $B\in\mathbb N_0$ a cost bound. Let $\langle\ctask, \rinit, \rgoal,
  \rtrans\rangle$ be the PB task encoding for $\Pi$ and $B$ and $s$ be a state over
  $\variables$.

  A heuristic certificate for state $s$ of $\Pi$ with bound $B$
  is a tuple $\langle\langle H, r^h_s\rangle, \proofhstate{s}, \proofhstateinductive{s}, \proofhstategoal{s}\rangle$, where
  \begin{itemize}
    \item $\langle H, r^h_s\rangle$ is a PB circuit with input variables $\variables
      \cup\mmvars$ not mentioning a primed variable.
    \item \textbf{state lemma:} $\proofhstate{s}$ is a CPR proof for\\
      $\mathcal C_s\cup H\cup\cgeq\cup\ccost \models (r_s\land \gecost{\max\{0,B-h(s)\}})
      \rightarrow r_s^h$, where $\mathcal C_s = \{r_s
      \Leftrightarrow \sum_{v\in s} v + \sum_{v\in\variables\setminus s}\bar v \geq
      |\variables|\}$.
      \item \textbf{goal lemma:} $\proofhstategoal{s}$ is a CPR proof for\\
        $\cgoal\cup H\cup\cgeq\cup\ccost \models (\rgoal\land{r_s^h})\rightarrow
        {\gecost{B}}$.
    \item \textbf{inductivity lemma:} $\proofhstateinductive{s}$ is a CPR proof for\\
      $\ctrans\cup H\cup H'\cup\cgeq\cup\ccost \models (r_s^h \land\rtrans)\rightarrow {r'}_s^h$,\\
      where $H'$ is a copy of $H$ where all PB variables are replaced
      with their primed version.
  \end{itemize}
\end{definition}

As mentioned earlier, heuristic certificates can be seen as certificates
for state $s$ where some cost of at least $B-h(s)$ has already been spent to
reach $s$. Thus it is no coincidence that these heuristic certificates also
structurally resemble lower-bound certificates under cost bound $B$.

We will later showcase for pattern database heuristics and $\hmax$ how such
heuristic certificates can be generated. But first we show how they contribute
to the overall lower-bound certificate generated by $\astar$.

\subsection{Proof-Logging \astar}

The invariant from a proof-logging $\astar$ search will conceptually cover
two aspects:
(1) the invariant is true for all closed states with their corresponding
cost from the initial state. For these states, the heuristic estimates
cannot rule out that they could be traversed with the corresponding cost by an optimal plan.
(2) the invariant is also true for any state-cost pair for which any
invariant is true that was  produced by the heuristic to actually rule states out.
The first  part considers the distances from the initial state to the states in the closed list, while
the second part considers the distances from the states in the open list to a goal.

Proof-logging $\astar$ maintains a sequence $A$ of reifications and a proof
log $L$ of derivations.  During its initialization, the heuristic already writes
some information to $A$ and $L$.

Whenever $\astar$ removes a state $s$ with g-value $g$ from the open list, it
adds a reification
\begin{equation}
\rstateming{s}{g} \Leftrightarrow \sum_{v\in s}
v + \sum_{v\in\variables\setminus s} \bar v + \gecost{g} \geq |\variables| + 1
\label{eq:stateming}
\end{equation}
to $A$, characterizing all pairs $\langle s,\tilde g\rangle$ with $\tilde g \geq g$.
In addition, it adds $\langle s,g\rangle$ to the initially empty collection $\closed$. It
also logs some derivations in $L$ that we describe later. If a successor 
$s'$ with g-value $g'$ is not added to $\open$ because $\astar$'s duplicate
detection is aware of an earlier encounter with $g''\leq g'$, it uses Lemma~\ref{lemma:costs1-variant} and logs
\begin{equation}
  \mmgeprimedcost{g'}\rightarrow\mmgeprimedcost{g''}.
  \label{eq:duplicatedetection}
\end{equation}

Whenever the search uses the heuristic to evaluate a state, the heuristic can
extend $A$ with further reifications, adds derivations for the corresponding
state lemma, inductivity lemma and goal lemma to $L$ and returns the
corresponding PB variable $r_s^h$ to the search. In addition, we require it to
log a proof for the state lemma in terms of the primed variables.

When the search terminates, it adds to $A$ a reification
\begin{equation}
r_{\astar} \Leftrightarrow \sum_{\langle s,g\rangle\in \closed} \rstateming{s}{g}
  + \sum_{\langle s,g,h\rangle\in\open} r_s^h \geq 1
\label{eq:rastar}
\end{equation}
and prepends $A$ with the necessary reifications \eqref{eq:mmgecost} and
\eqref{eq:mmgeprimedcost}, where $B$ is the cost of the found plan.

In the following we explain how we can generate the three proofs $\proofinit,
\proofinductive$ and $\proofgoal$ for the PB circuit $\langle A, r_{\astar}\rangle$.
This generation relies on the correctness of \astar.
However, the verifier receiving the generated proofs relies only on the correctness of the derivation rules of CPR.
We start by showing how $L$ can be extended to $\proofinit$.

\begin{lemma}
  RUP can derive the initial state lemma
  $(\rinit \land \overline{\gecost{1}})\rightarrow r_{\astar}$ from $\cinit\cup A\cup\cgeq$.
  \label{lemma:astarinitialstate}
\end{lemma}

\begin{proof}
  Assume (a) $\rinit\geq 1$, (b) $\overline{\gecost{1}}\geq 1$, and (c)
  $\overline{r_{\astar}}\geq 1$. From (a) and \eqref{eq:rinit}, we receive (d)
  $v\geq1$ for each $v\in\init$ and $\overline{v}\geq 1$ for each $v\in\variables\setminus\init$.
  Reification \eqref{eq:gecost} and (b) express $\sum\nolimits_{i=0}^{\maxbit}
  2^ic_i<1$, from them we get for all $c_i\in\cvars$ that $\overline{c_i}\geq
  1$. With \eqref{eq:gecost}, we receive $\gecost{0}$, which with (d) and
  \eqref{eq:stateming} gives $\rstateming{\init}{0}\geq 0$.
  Since $\langle\init,0\rangle$ is in $\closed$, with (a) and \eqref{eq:rastar}
  we get $r_{\astar}\geq 1$, contradicting (c).
\end{proof}

The proof $\proofgoal$ builds on the goal lemmas logged by the heuristic:
\begin{lemma}
  It is possible to derive the goal lemma $(\rgoal\land r_{\astar})\rightarrow
  \gecost{B}$ from $\cgoal\cup A\cup\cgeq$.
\end{lemma}

\begin{proof}
  The heuristic certificate can derive the goal lemma for all open states.
  The rest is then by RUP, assuming (a) $\rgoal\geq 1$, (b) $r_{\astar}\geq 1$
  and (c) $\overline{\gecost{B}}\geq 1$.

  From \eqref{eq:rgoal} and (a) we get that (d) $v\geq 1$ for all goal variables.
  For all closed pairs $\langle s,g\rangle$, $s$ is not a goal state or it is the reached
  goal state $s_{\star}$ with $g=B$.
  If $s$ is not a goal state, then some $v\in\goal$ is false in $s$ and 
  \eqref{eq:stateming} yields with (d) that $\overline{\rstateming{s}{g}}\geq
  1$. We also get $\overline{\rstateming{s_{\star}}{B}}\geq 1$ from
  \eqref{eq:stateming}, using (c).

  For all variables $r_s^h$, the goal lemma provided by the heuristic implies
  with (a) and (c) that $\overline{r_s^h}\geq 1$, so overall we get with
  \eqref{eq:rastar} that $\overline{r_{\astar}}\geq 1$, contradicting (b).
\end{proof}

To support the derivation of the inductivity lemma, $\astar$ already extends
$L$ upon every expansion for every action with the derivation described in the
following lemma:
\begin{lemma}
    For every action $a$ applicable in $s$ and state $s$ closed with cost $g$,
  it is possible to derive $(\rstateming{s}{g} \land r_a)\rightarrow r'_{\astar}$
  from $\ctrans\cup A\cup A'\cup\cgeq$.
  \label{lemma:expandedaction}
\end{lemma}
\begin{proof}
  We first establish by Lemma~\ref{lemma:costs3} that
  (a) $(\gecost{g}\land \deltacost{\cost(a)})\rightarrow\geprimedcost{\min\{B,g+\cost(a)\}}$.

  The desired constraint follows by RUP:
  Assume
  (b) $\rstateming{s}{g}\geq 1$,
  (c) $r_a\geq 1$, and
  (d) $\overline{r'_{\astar}}\geq 1$.

  From (b), we can derive with \eqref{eq:stateming} that
  (e) $v\geq 1$ for all $v\in s$,
  (f) $\overline{v}\geq 1$ for all $v\in\variables\setminus s$, and
  (g) $\gecost{g}\geq 1$.
  We use (c) with \eqref{eq:actionvar} to derive
  (h) $\deltacost{\cost(a)}\geq1$.
  With (h), (g), and (a) we get
  (i) $\geprimedcost{\min\{B,g+\cost(a)\}}\geq 1$.
  If $s$ is a goal state then $g=B$ and we get from (i) and \eqref{eq:actionvar} that
  $\overline{r_a}\geq 1$, contradicting (c). 

  Otherwise, 
  we
  define $\tilde{s}=\app{s}{a}$
  and use \eqref{eq:actionvar}, (c), (e), (f), (g), and \eqref{eq:eqvars}
  to derive that
  (j) the primed state variables encode $\tilde{s}$ and that
  (k) $\mmgeprimedcost{g+\cost(a)}\geq 1$.

  If the successor was not considered because of a duplicate $\langle
  \tilde{s},\hat{g}\rangle$ with $\hat{g}<g+\cost(a)$, we use the derived constraint 
  \eqref{eq:duplicatedetection} to derive (k') $\mmgeprimedcost{\hat{g}}\geq 1$.

  If $\langle\tilde{s},g+\cost(a)\rangle\in\closed$ or
  $\langle\tilde{s},\hat{g}\rangle\in\closed$, respectively, we use the primed
  version of \eqref{eq:stateming} with (j) and (k) or (k') to derive
  ${\rstateming{\tilde{s}}{g+\cost(a)}}'\geq 1$ or
  ${\rstateming{\tilde{s}}{\hat{g}}}'\geq 1$.
  With the primed version of \eqref{eq:rastar} we derive $r'_{\astar}\geq 1$,
  contradicting (d).

  Otherwise, there is an entry for $\tilde{s}$ in $\open$.
  Since $\astar$ closes all states with $f < B$, we know
  that in this case $g+\cost(a)+h(\tilde{s})\geq B$,
  so $g+\cost(a)\geq \max\{0, B - h(\tilde{s})\}$. We can thus use (j) and (k)
  with the primed version of the state lemma for
  $\tilde{s}$ from the heuristic and receive ${r_{\tilde{s}}^h}'\geq 1$.
  The primed version of \eqref{eq:rastar} yields $r'_{\astar}\geq 1$,
  contradicting (d).
\end{proof}

\begin{lemma}
  For every state $s$ closed with cost $g$, it is possible to derive
  $(\rstateming{s}{g} \land \rtrans )\rightarrow r'_{\astar}$ from $\ctrans\cup A\cup A'\cup\cgeq$.
  \label{lemma:expandedtransition}
\end{lemma}
\begin{proof}
  We establish by Lemma~\ref{lemma:expandedaction} that 
    (a)
  $(\rstateming{s}{g} \land r_a)\rightarrow r'_{\astar}$ for every action
  $a\in \actions$ that is applicable in $s$.

  Then RUP derives the constraint:
  Assume 
  (b) $\rstateming{s}{g}\geq 1$,
  (c) $\rtrans\geq 1$, and
  (d) $\overline{r'_{\astar}}\geq 1$.
  With (a), (b) and (d) we derive $\overline{r_a}\geq 1$ for each action
  $a$ applicable in $s$. %
  From (b), we can derive with \eqref{eq:stateming} that
  (e) $v\geq 1$ for all $v\in s$,
  (f) $\overline{v}\geq 1$ for all $v\in\variables\setminus s$, and
  if $a$ is not applicable in $s$, some precondition is violated 
  and we use (f) and \eqref{eq:actionvar} to derive $\overline{r_a}\geq 1$.
  With \eqref{eq:rtrans}, this yields $\overline{\rtrans}\geq 1$,
  contradicting (c).
\end{proof}

These derivations are already logged during the execution of the search. We
extend the log with derivations from the following lemma to generate
$\proofinductive$:

\begin{lemma}
  It is possible to derive the inductivity lemma $(r_{\astar}\land
  \rtrans)\rightarrow r'_{\astar}$ from $\ctrans\cup A\cup A'\cup\cgeq$.
  \label{lemma:astarinductivity}
\end{lemma}

\begin{proof}
  We first derive by RUP for every $s$ with some $\langle s,g,h\rangle\in\open$ that
  (a) $(r_s^h \land \rtrans )\rightarrow r'_{\astar}$: From the assumption
  $\overline{r'_{\astar}}\geq 1$, we get $\overline{{r'}_s^h}\geq 1$. Using the
  inductivity lemma from the heuristic for $s$ with the other assumptions
  $r_s^h\geq 1$ and $\rtrans\geq 1$, we get the contradiction.
  
  For every state $s$ expanded with cost $g$, we get
  (b) $(\rstateming{s}{g} \land \rtrans )\rightarrow r'_{\astar}$ as described
  in Lemma~\ref{lemma:expandedtransition}.

  Afterwards RUP can derive the constraint, assuming (c) $r_{\astar}\geq 1$,
  (d) $\rtrans\geq 1$, and (e) $\overline{r'_{\astar}}\geq 1$.

  For every open state $s$, we get from (a,d,e) that (f) $\overline{r_s^h}\geq 1$.
  For every state $s$ closed with g-value $g$, we get from (b, d, e) that (g)
  $\overline{\rstateming{s}{g}}\geq 1$. With \eqref{eq:rastar}, we get from (f)
  and (g) that $\overline{r_{\astar}}\geq 1$, contradicting (c).
\end{proof}

To analyze the overhead of proof-logging $\astar$, we assume that each
proof from the heuristic is provided with only a constant-factor
overhead to the heuristic computation time.
Logging a single reification constraint \eqref{eq:stateming} is an operation that requires
time linear in $|\variables|$.
A constraint of this kind has to be logged for each closed state.
This is still a constant-factor overhead because the closed state has to be generated first,
which is an operation with time linear in $|\variables|$, too.
The reification of \eqref{eq:rastar} is linear in the number of states generated.
Both the goal lemma and the initial state lemma 
are logged by single RUP statements of constant size.

In the proof of the inductivity lemma we first use as many constant-size statements as there are open states.
The second RUP subproof in the inductivity lemma requires the constraint from
  Lemma~\ref{lemma:expandedtransition}
   specified on each state $s\in\closed$, which in turn requires the constraint from 
   Lemma~\ref{lemma:expandedaction} specified on each applicable
   action in $s$.
This amortizes with the generation of the successors of $s$ when $s$ moves from $\open$ to $\closed$. 
We see that in total there is a constant-factor overhead to the proof-logging of $\astar$.

\subsection{Proof-Logging Pattern Database Heuristics}

Abstraction heuristics such as \emph{pattern database (PDB) heuristics} use
goal distances in an induced abstract task for the heuristic estimates.
Let $\Pi = \langle \variables, \actions, \init, \goal\rangle$ be a STRIPS
planning task.  A PDB heuristic is defined in terms of a \emph{pattern}
$P\subseteq \variables$ and its abstraction function $\alpha$ maps each state
$s$ over $\variables$ to an abstract state $\alpha(s)$ over $P$ as $\alpha(s)
= s\cap P$.  Each action $a\in\actions$ induces the abstract action $\aalpha$
with $\pre(\aalpha) = \pre(a)\cap P$, $\add(\aalpha) = \add(a)\cap P$,
$\del(\aalpha) = \del(a)\cap P$, and $\cost(\aalpha) = \cost(a)$. The abstract
goal $\goal^\alpha$ is $\goal\cap P$. The heuristic estimate of the PDB for
state $s$ is the cost of an optimal solution of the abstract task
$\langle P, \{a^\alpha\mid a\in \actions\}, \alpha(s), \goal^{\alpha}\rangle$
or $\infty$ if it is unsolvable. This is a lower bound of the goal distance of
$s$ in $\Pi$, as any solution of the concrete task corresponds to a solution of
the abstract task with equal cost.
In practice, a PDB heuristic does not build an abstract task for each heuristic
evaluation but precomputes the abstract goal distances $d(\pdbstate)$ for all
abstract states $\pdbstate$ and stores them in a so-called pattern database. When
a concrete state $s$ is evaluated, the PDB heuristic computes $\alpha(s)$ and
returns the stored goal distance $d(\alpha(s))$ as the heuristic estimate $h(s)$.

The invariant for a PDB heuristic should hold for all pairs $\langle s,c\rangle$
such that the abstract goal distance $d$ of $\alpha(s)$ is already so high
that $c+d\geq B$, and thus it is impossible to reach the goal from $s$ with
a strictly lower cost than $B$ if reaching $s$ already costs $c$.

Let $S_\alpha$ be the set of all abstract states.
For each $\pdbstate \in S_\alpha$, the heuristic introduces two PB variables.
The variable $\rpdbstate{\pdbstate}$ is true iff the variables from $\variables$ encode
a state $s$ with $\alpha(s) = \pdbstate$ by adding a reification
\begin{equation}
 \rpdbstate{\pdbstate}\Leftrightarrow \sum\nolimits_{v\in \pdbstate}v
+  \sum\nolimits_{v\in P\setminus\pdbstate}\bar{v}\geq
|P|.\label{eq:rinefficientpdbstate}
\end{equation}
The variable
$\rpdbstategeg{\pdbstate}$ is true iff the variables from $\variables$ and
$\cvars$ encode a pair $\langle s,c\rangle$ such that $\alpha(s) =\pdbstate$ and $c\geq
\max\{B-d(\pdbstate),0\}$ by adding reification
\begin{equation}
 \rpdbstategeg{\pdbstate}\Leftrightarrow \rpdbstate{\pdbstate}
+ \mmgecost{B-d({\pdbstate})}\geq 2.\label{eq:rinefficientpdbstategeg}
\end{equation}
As a final reification for the invariant the heuristic adds
\begin{equation}
r_{\textup{PDB}} \Leftrightarrow \sum\nolimits_{\pdbstate\in S_{\alpha}}
\rpdbstategeg{\pdbstate} \geq 1.
\label{eq:rinefficientPDB}
\end{equation}

The PB circuit for any state then is $\langle\Hpdb, r_{\textup{PDB}}\rangle$, where $\Hpdb$ is the
sequence of all reifications \eqref{eq:rinefficientpdbstate}, \eqref{eq:rinefficientpdbstategeg} and \eqref{eq:rinefficientPDB}.
On the evaluation of a state $s$, the heuristic always returns reification variable
$r_{\textup{PDB}}$ to the search. In the following we discuss how the required
proofs can be generated.
This generation relies on the correctness and admissibility of the PDB heuristic,
but the generated proofs themselves do not.

We begin with the state lemma, which requires a new proof for each evaluated
state.

\begin{restatable}{lemma}{pdbstatelemma}
RUP can derive the state lemma\\
 $(r_s\land \gecost{\max\{0,B-h(s)\}})
    \rightarrow r_{\textup{PDB}}$ from $\mathcal C_s\cup \Hpdb\cup\cgeq\cup\ccost$,
where $\mathcal C_s = \{r_s
\Leftrightarrow \sum_{v\in s} v + \sum_{v\in\variables\setminus s}\bar v \geq
|\variables|\}$.
\label{lemma:inefficientPDBstatelemma}
\end{restatable}

Since the PB circuit for the heuristic is the same for all evaluated states, we
only need to include the proof for the goal lemma and for the inductivity lemma
once in the overall generated proof. For the goal lemma, we use the following:

\begin{restatable}{lemma}{pdbgoallemma}
RUP can derive the goal lemma\\
$(\rgoal\land{r_{\textup{PDB}}})\rightarrow  {\gecost{B}}$ from $\cgoal\cup \Hpdb\cup\cgeq\cup\ccost$.
\label{lemma:inefficientPDBgoallemma}
\end{restatable}

For the inductivity lemma, we develop the derivation by means of four lemmas. The
first one derives that applying an induced abstract action
$\aalpha$ in abstract state $\pdbstate$ leads to the abstract successor state.

\begin{restatable}{lemma}{pdbinductivitylemmaone}
For each action $a\in \actions$ and abstract state $\pdbstate$ such that $\aalpha$
is applicable in $\pdbstate$, RUP can derive\\
$(\rpdbstate{\pdbstate}\land r_a)\rightarrow
  \rpdbprimedstate{\app{\pdbstate}{a^\alpha}}$ from $\ctrans\cup \Hpdb\cup \Hpdb'\cup\cgeq\cup\ccost.$
\label{lemma:inefficientPDB1}
\end{restatable}

The next lemma again considers individual actions and abstract states, but
takes cost into account and derives that the invariant of the certificate is
true for the successor state-cost pair.
\begin{lemma}
For each action $a\in \actions$ and abstract state $\pdbstate$ such that $\aalpha$
is applicable in $\pdbstate$, it is possible to derive
$(\rpdbstategeg{\pdbstate}\land r_a)\rightarrow
  r'_{\textup{PDB}}$ from $\ctrans\cup \Hpdb\cup \Hpdb'\cup\cgeq\cup\ccost$.
\label{lemma:inefficientPDB2}
\end{lemma}

\begin{proof}
We start by deriving some constraints over costs that we will use later in a RUP
proof.

First, we derive (a)
  $(\mmgecost{B-d(\pdbstate)}\land\deltacost{\cost(a)})\rightarrow\mmgeprimedcost{B-d(\pdbstate)+\cost(a)}$
as described in Lemma~\ref{lemma:costs3}.

We know that for the abstract goal distance it holds that
$d(\app{\pdbstate}{a^\alpha})+\cost(a^\alpha)\geq d(\pdbstate)$.
Together with $\cost(a^\alpha) =\cost(a)$, we get
$B-d(\app{\pdbstate}{a^\alpha})\leq B-d(\pdbstate) + \cost(a)$.
We use this to derive (b)
$\mmgeprimedcost{B-d(\pdbstate)+\cost(a)}
\rightarrow\mmgeprimedcost{B-d(\app{\pdbstate}{a^\alpha})}$ as described
in Lemma~\ref{lemma:costs1-variant}.
In addition, we establish (c) $(\rpdbstate{\pdbstate}\land r_a)\rightarrow
\rpdbprimedstate{\app{\pdbstate}{a^\alpha}}$ by RUP (Lemma~\ref{lemma:inefficientPDB1}).

Now we can derive the constraint in the claim by RUP, assuming that
(d) $\rpdbstategeg{\pdbstate}\geq 1$, (e) $r_a\geq 1$, and
(f) $\overline{r'_{\textup{PDB}}}\geq 1$.
From (d), we get with \eqref{eq:rinefficientpdbstategeg} that (g) $\rpdbstate{\pdbstate}\geq 1$
and (h) $\mmgecost{B-d({\pdbstate})}\geq 1$. From (e),(g) and (c) we
derive (i) $\rpdbprimedstate{\app{\pdbstate}{a^\alpha}}\geq 1$.

From (e) and \eqref{eq:actionvar}, we derive (j) $\deltacost{\cost(a)}\geq 1$.
From (a) together with (h) and (j) we derive that (k)
$\mmgeprimedcost{B-d(\pdbstate)+\cost(a)}\geq 1$.
From (b) together with (k) we derive that (l)
$\mmgeprimedcost{B-d(\app{\pdbstate}{a^\alpha})}\geq1$.
We get with this, (i) and the primed constraint
\eqref{eq:rinefficientpdbstategeg} from $\Hpdb$ that
$\rpdbprimedstategeg{\app{\pdbstate}{a^\alpha}}$. With the primed version of
\eqref{eq:rinefficientPDB} from $\Hpdb'$ this yields $r'_{\textup{PDB}}\geq 1$,
contradicting (f).
\end{proof}

The third lemma generalizes the previous lemma from individual actions to the
entire transition relation.
\begin{restatable}{lemma}{pdbinductivitylemmatwo}
For each abstract state $\pdbstate$ it is possible to derive
$(\rpdbstategeg{\pdbstate}\land \rtrans)
  \rightarrow r'_{\textup{PDB}}$ from $\ctrans\cup \Hpdb\cup \Hpdb'\cup\cgeq\cup\ccost$.
\label{lemma:inefficientPDB3}
\end{restatable}

\begin{proof}[Proof sketch]
  For each action $a\in \actions$ that is applicable in $\pdbstate$,
  establish $(\rpdbstategeg{\pdbstate}\land r_a)\rightarrow
  r'_{\textup{PDB}}$ with Lemma~\ref{lemma:inefficientPDB2}.
  Then the desired constraint can be derived by RUP.
\end{proof}

The final step for the inductivity lemma generalizes this from individual
abstract states to the entire PB circuit of the heuristic.
\begin{restatable}{lemma}{pdbinductivitylemma}
  It is possible to derive the inductivity lemma $(r_{\textup{PDB}} \land\rtrans)\rightarrow
  r'_{\textup{PDB}}$ from $\ctrans\cup \Hpdb\cup
  \Hpdb'\cup\cgeq\cup\ccost$.
  \label{lemma:inefficientPDBinductivitylemma}
\end{restatable}

\begin{proof}[Proof sketch]
  Establish 
  $(\rpdbstategeg{\pdbstate}\land \rtrans) \rightarrow r'_{\textup{PDB}}$
  for each abstract state $\pdbstate$ 
  by Lemma~\ref{lemma:inefficientPDB3}. The rest follows by RUP.
\end{proof}

We considered PDB heuristics as one example of abstractions because they have
a particularly simply structured abstraction function. To adapt the approach to
other abstraction heuristics, the overall line of argument still works but we
would have to replace \eqref{eq:rinefficientpdbstate} with one or more reifications that
allow to identify the concrete states that correspond to a given abstract
state. In addition, the derivations from Lemmas~\ref{lemma:inefficientPDBstatelemma},
\ref{lemma:inefficientPDBgoallemma} and \ref{lemma:inefficientPDB1} need to be adapted
accordingly.

The described proof logging only has a constant-factor overhead for a somewhat
naive implementation of PDBs. While the state and goal lemma only require
a single RUP statement, the inductivity lemma requires effort in
$O(|S_\alpha||\actions|)$, iterating once over all actions for each abstract
state. A better PDB implementation can deal more efficiently with abstract
states that are not reverse-reachable.
In the
\ifthenelse{\boolean{icaps}}{extended version of this paper}{appendix}
we describe a variant of the certificate that
does not explicitly represent such states and which guarantees
a constant-factor overhead also for such a more efficient computation of PDBs.

\subsection{Proof-Logging \hmax}

The maximum heuristic \cite{bonet-geffner-aij2001} is based on
a relaxation of the planning task that makes two simplifying assumptions:
first, actions do not delete variables; second, the cost of making a set of
variables true corresponds to the cost of making its most expensive variable
true; since action preconditions are sets of variables, the relaxation affects
the cost of enabling action applications, which in turn affects the cost to
make the individual variables true. For this reason, the heuristic is
mathematically specified by a system of equations.
The maximum heuristic $\hmax(s)$ is $\hmax(s,G)$, where $\hmax(s,V)$ is the
pointwise greatest function with $\hmax(s,V) =\\
\begin{cases}
  0 &\text{if $V\subseteq s$}\\
  \min\limits_{\substack{a\in \actions,\\v\in \add(a)}}(\cost(a)
    + \hmax(s,\pre(a)))
    &\stackunder{\text{if $|V|=1$}}
                {\text{and $V\not\subseteq s$}}\\
  \max\limits_{v\in V}\hmax(s,\{v\}) & \text{otherwise}
\end{cases}$

For $v\in \variables$, the \emph{max value} is $\vmax(v) = \hmax(s,\{v\})$,
which can be seen as the cost of making $v$ true (starting from $s$) under this
relaxation. The heuristic certificate will build on the following insight:
since $\hmax$ is admissible, it will be impossible to reach the goal from $s$
with overall cost $< B$ if the cost to reach $s$ already exceeds $B
- \hmax(s)$. If in a state $\hat s$ additional variables are true, the state
can be closer to the goal but under the relaxation of the heuristic the
advantage gained from an individual variable $v$ will be at most $\vmax(v)$, so
for arbitrary non-empty states $\hat s$, the goal cannot be reached if the cost
to reach $\hat s$ already exceeds $B - \hmax(s) + \max_{v\in \hat s}\vmax(v)$
(note that $\vmax(v) = 0$ for $v\in s$), even under the relaxation of the
heuristic. In the delete-relaxed task, the empty state cannot be closer to the
goal than $s$, so we can use bound $B - \hmax(s)$.

A typical efficient implementation of the maximum heuristic only
determines the max value for all $v$ with $\vmax(v) < \hmax(s)$. We can still
build a valid heuristic certificate, requiring that the cost $B$ has already
been exceeded for states that contain a variable for which
$\vmax(v)\geq\hmax(s)$. For the formal definition of the certificate, we will
handle this aspect by an auxiliary value $\boundedvmax(v)
= \min\{\vmax(v),\hmax(s)\}$.

For each $v\in \variables$, the heuristic adds a reification:
\begin{equation}
  r_{v,s}\Leftrightarrow \bar
  v + \mmgecost{B-\hmax(s)+\boundedvmax(v)} \geq 1
\label{eq:rmaxv}
\end{equation}

In addition, it adds the following reification:

\begin{equation}
\rmax_s \Leftrightarrow \mmgecost{B-\hmax(s)} + \sum_{v\in \variables} r_{v,s} \geq
  |\variables| + 1
\label{eq:rmaxs}
\end{equation}

The PB circuit for state $s$ then is $\langle\Hmax, \rmax_s\rangle$, where $\Hmax$ is the
sequence of all reifications \eqref{eq:rmaxv} and \eqref{eq:rmaxs}.

Intuitively, \eqref{eq:rmaxv} expresses the implication that the cost is at least
$B-\hmax(s)+\boundedvmax(v)$ if $v$ is true. Equation \eqref{eq:rmaxs}
corresponds to a conjunction of $\mmgecost{B-\hmax(s)}$ and such implications
for each variable, so overall $\rmax_s$ will be true for all state value pairs
$\langle \hat s, \hat c\rangle$, where $\hat s\neq\emptyset$ and $\hat c\geq
\max_{v\in \hat s}\{B-\hmax(s)+\boundedvmax(v)\}$ or where $\hat s=\emptyset$
and $\hat c\geq B-\hmax(s)$.

\begin{restatable}{lemma}{hmaxstatelemma}
RUP can derive the state lemma\\
$(r_s\land \gecost{\max\{0,B-\hmax(s)\}})
    \rightarrow \rmax_s$ from $\mathcal C_s\cup \Hmax\cup\cgeq\cup\ccost$,
where $\mathcal C_s = \{r_s
\Leftrightarrow \sum_{v\in s} v + \sum_{v\in\variables\setminus s}\bar v \geq
|\variables|\}$.
\label{lemma:hmaxstate}
\end{restatable}

A single RUP statement is sufficient for the goal lemma:
\begin{restatable}{lemma}{hmaxgoallemma}
  RUP can derive the goal lemma\\
  $(\rgoal\land{\rmax_s})\rightarrow
  {\gecost{B}}$ from $\cgoal\cup \Hmax\cup\cgeq\cup\ccost$.
\end{restatable}

For the inductivity lemma, we first show an analogous statement for a single action.

\begin{restatable}{lemma}{hmaxinductivityaux}
For every action $a$, it is possible to derive\\
$(\rmax_s \land r_a)\rightarrow {\rmax_s}'$ from $\ctrans\cup \Hmax\cup \Hmax'\cup\cgeq\cup\ccost$.
  \label{lemma:hmaxaction}
\end{restatable}

\begin{proof}[Proof sketch]
  The overall argument of the proof is as follows:

  If the action has a precondition $p$, we can conclude that $p\geq 1$, which
  allows us to derive with \eqref{eq:rmaxv} that
  $\mmgecost{B-\hmax(s)+\boundedvmax(p)}\geq 1$. We can establish
  $\mmgeprimedcost{B-\hmax(s)}$ using Lemmas \ref{lemma:costs1-variant} and
  \ref{lemma:costs3}.

  To establish $r'_{v,s}\geq 1$ for all variables, we consider the precondition
  with maximal $\boundedvmax$ value.

  If $\boundedvmax(p) = \hmax(s)$, then
  $\mmgecost{B-\hmax(s)+\boundedvmax(p)}\geq 1$ corresponds to
  $\mmgecost{B}\geq 1$. We derive $\mmgeprimedcost{B}\geq 1$, which is
  sufficient to establish $r'_{v,s}\geq 1$.

  If $\boundedvmax(p) < \hmax(s)$, we
  establish $r'_{v,s}\geq 1$ depending on whether and how $v$ occurs in the
  effect of $a$. If it is a delete effect, then $\overline{v'}\geq 1$. If it is an
  add effect, we exploit that $p$ has maximal $\boundedvmax$ value among the
  preconditions, so from the definition of the maximum heuristic we can see
  that $\boundedvmax(v)\leq \boundedvmax(p) +\cost(a)$. Establishing
  $\mmgeprimedcost{B-\hmax(s)+\boundedvmax(p)+\cost(a)}$ and using this insight
  gives $\mmgeprimedcost{B-\hmax(s)+\boundedvmax(v)}\geq 1$. If
  $v\notin\effvars(a)$, we can carry over the reason for $r_{v,s}\geq 1$ to the
  primed variables.

  If the precondition of the action is empty, we build on
  $\mmgecost{B-\hmax(s)}$, which gives $\mmgeprimedcost{B-\hmax(s)+\cost(a)}$
  and $\mmgeprimedcost{B-\hmax(s)}$ using Lemmas \ref{lemma:costs1-variant} and
  \ref{lemma:costs3}. For establishing $r'_{v,s}\geq 1$ for all variables, we
  proceed as before for the case $\boundedvmax(p) < \hmax(s)$ (replacing
  $\boundedvmax(p)$ with $0$). For the special case $\hmax(s)=0$,
  we can instead use the shorter argument as in the previous case
  $\boundedvmax(p) = \hmax(s)$, building on $\mmgecost{B}\geq 1$.

  The technical details are included in
  the\ifthenelse{\boolean{icaps}}{ extended version of this paper
    \cite{dold-et-al-arxiv2025}.}{ appendix.}
\end{proof}

The full inductivity lemma follows quite directly.
\begin{restatable}{lemma}{hmaxinductivitylemma}
It is possible to derive the inductivity lemma\\
$(\rmax_s \land \rtrans)\rightarrow
  {\rmax_s}'$ from $\ctrans\cup \Hmax\cup \Hmax'\cup\cgeq\cup\ccost$.
  \label{lemma:hmaxinductivity}
\end{restatable}
\begin{proof}[Proof sketch]
  Establish with Lemma~\ref{lemma:hmaxaction} for every $a\in\actions$ that
  $(\rmax_s \land r_a)\rightarrow {\rmax_s}'$. Then derive the constraint
  by RUP.
\end{proof}

Extending $\hmax$ with proof logging only leads to the following overhead:
adding reifications \eqref{eq:rmaxv} for each $v\in\variables$ and
\eqref{eq:rmaxs} takes additional time linear in $|\variables|$ because
$h(s)$ and the values necessary for $\boundedvmax{}$ are already computed
during the normal evaluation. In addition, the implementation has to reinitialize some values
for each variable before each new evaluation, so the overhead is within
a constant factor.

The state and goal lemmas each only require a single RUP statement.
For the inductivity lemma, we need to establish the constraint from
Lemma~\ref{lemma:hmaxaction} for every action. For each action where
$\hmax(s,\pre(a))\geq\hmax(s)$, we require constant effort. For all other
actions, the effort for the action is linear in the number of variables.

\section{Conclusion}

We introduced lower-bound certificates for classical planning, which
can be used to verify the optimality of optimal planners.
The certificates are sound, complete and efficiently verifiable. We
showed them to be efficiently generatable in a case study on
\astar\ with \hmax\ or pattern database heuristics.

We believe these certificates to be much more general than the ones
considered in previous research because they work on a more
fundamental level of abstraction, build on an established proof
systems that has proved its utility in many other areas of AI, and are
able to directly incorporate numerical reasoning. In future work, this
generality has to be confirmed by applying these certificates to other
optimal planning approaches that current approaches cannot handle. In
particular, we believe that pseudo-Boolean constraints can generically
handle the concept of \emph{cost partitioning}
\cite{katz-domshlak-aij2010}, perhaps \emph{the} most important
optimal planning technique, because they are able to reason numerically
at the level of individual state transitions in a way that previously
proposed certificates cannot.

\section*{Acknowledgments}
This research was supported by the Swiss
National Science Foundation (SNSF) as part of the project
\dq{Unifying the Theory and Algorithms of Factored StateSpace Search} (UTA).
Jakob Nordström was supported by the Independent Research Fund Denmark grant \mbox{9040-00389B}.

\ifthenelse{\boolean{icaps}}
    {} 
    {\clearpage

\appendix

\section{Appendix}

The following sections contain additional details not provided in the
ICAPS 2025 version of this paper.

\section{Proofs for Cost Lemmas}

We will use an extended cutting planes proof technique that we did not
introduce in the background section, namely redundance-based strengthening
(RED): if for a PB constraint $C$ there is a substitution $\rho$ such that
$\mathcal C\cup\{\lnot C\}\models \mathcal{C}\cup\{
C\}\!\!\upharpoonright_\rho$ then we can infer $C$. We only use the simple form
with an empty substitution. Since it cannot invalidate the core formula
$\mathcal C$, we only need to derive that $\mathcal C\cup\{\lnot C\}\models
 C$ to establish $C$. We can think of this as a more flexible form of
a proof by contradiction than simple RUP.

We now provide the proofs for the two cost-related lemmas from the paper.

\costsone*

\begin{proof}
  By RED with an empty witness. Since the witness is empty, all constraints
  from $\cgeq$ are trivially implied, so we only need to derive
  $C\dot=\gecost{\min\{B,\max\{0, j+k\}\}}\rightarrow\gecost{\min\{B,\max\{0,j\}\}}$.

  We may assume the negation of this constraint, which gives us by weakening
  that (a) $\gecost{\min\{B,\max\{0, j + k\}\}}\geq 1$ and (b)
  $\overline{\gecost{\min\{B,\max\{0,j\}\}}}\geq 1$. Using \eqref{eq:gecost},
  we get from (a) that (c) $\sum\nolimits_{i=0}^{\maxbit} 2^ic_i \geq
  \min\{B,\max\{0,j+k\}\}$.

  Define $\Delta = \min\{B,\max\{0,j+k\}\} - \min\{B,\max\{0,j\}\}$ and observe
  that $\Delta\in \{0,\dots,B\}$.
  
  Let set $M$ contain exactly the numbers $i$ such that the $i$th
  bit in the binary representation of $\Delta$ is $1$.
  With a full weakening\footnote{Full weakening is syntactic sugar that corresponds to
  multiplying the literal axioms $\ell \geq 0$ with the coefficients of $\ell$
  and adding this to the constraint.} of (c) with $M$, we receive
  (d) $\sum\nolimits_{i=\in\{0,\dots,\maxbit\}\setminus M} 2^ic_i \geq
  \min\{B,\max\{0,j\}\}$. 

  For every $i\in M$, we can multiply the axiom lemma
  $c_i\geq 0$ with $2^i$ and receive $2^ic_i\geq 0$. The sum over all these
  constraints and (d) is $\sum\nolimits_{i=0}^{\maxbit} 2^ic_i \geq
  \min\{B,\max\{0,j\}\}$. With \eqref{eq:gecost}, this yields
  $\gecost{\min\{B,\max\{0,j\}\}}$.
\end{proof}

\coststwo*

\begin{proof}
  By RED with an empty witness. Since the witness is empty, all constraints
  from $\cgeq$ are trivially implied, so we only need to derive
  $C\dot=(\gecost{\min\{B,\max\{0, l\}\}}\land \deltacost{m})\rightarrow\geprimedcost{\min\{B,\max\{0,l+m\}\}}$.

  We may assume the negation of this constraint, which gives us by weakening
  that (a) $\gecost{\min\{B,\max\{0, l\}\}}\geq 1$, (b) $\deltacost{m}\geq 1$, and (c)
  $\overline{\geprimedcost{\min\{B,\max\{0,l+m\}\}}}\geq 1$.

  From (b) and \eqref{eq:deltacost} we get that (d)
  $\sum\nolimits_{i=0}^{\maxbit} 2^ic'_i - \sum\nolimits_{i=0}^{\maxbit} 2^ic_i
  \geq m$. From (a) and \eqref{eq:gecost}, we derive (e)
  $\sum\nolimits_{i=0}^{\maxbit} 2^ic_i \geq \min\{B,\max\{0,l\}\}$. The sum of (d) and
  (e) is (f) $\sum\nolimits_{i=0}^{\maxbit} 2^ic'_i\geq
  \min\{B,\max\{0,l\}\}+m$.  From (c), we derive (g)
  $\sum\nolimits_{i=0}^{\maxbit} 2^ic'_i < \min\{B,\max\{0,l+m\}\}$, which in
  normalized form is $\sum\nolimits_{i=0}^{\maxbit} 2^i\overline{c'_i} \geq
  \sum\nolimits_{i=0}^{\maxbit} 2^i - \min\{B,\max\{0,l+m\}\} + 1$.  The sum of
  (f) and (g) is (h) $0\geq \min\{B,\max\{0,l\}\}+m - \min\{B,\max\{0,l+m\}\} + 1$.
  Since $\min\{B,\max\{0,l\}\}+m
  \geq \min\{B,\max\{0,l+m\}\}$, the right-hand side of (h) is positive, so we
  have a contradiction, from which we can derive any constraint, thus also $C$.
\end{proof}

\section{Proof-Logging PDB Heuristics: Proofs}

\pdbstatelemma*

\begin{proof}
Assume (a) $r_s\geq 1$, (b) $\mmgecost{B-h(s)}\geq 1$, and (c)
$\overline{r_{\textup{PDB}}}\geq 1$.  From
(a) and $\mathcal C_s$, we receive $v\geq 1$ for each $v\in s$ and $\bar v\geq 1$ for each
$v\in\variables\setminus s$,
and in particular (d) $v\geq 1$ for each $v\in\alpha(s)$ and $\bar v\geq 1$ for each
$v\in P\setminus \alpha(s)$. With constraints \eqref{eq:rinefficientpdbstate} and (d), we
derive (e) $\rpdbstate{\alpha(s)}\geq 1$.  Since $h(s) = d(\alpha(s))$, we can
use (b) and (e) with \eqref{eq:rinefficientpdbstategeg} to derive (f)
$\rpdbstategeg{\alpha(s)}$. With \eqref{eq:rinefficientPDB}, this gives us
$r_{\textup{PDB}} \geq 1$, a contradiction to (c).
\end{proof}

\pdbgoallemma*

\begin{proof}
Assume (a) $\rgoal\geq 1$, (b) $r_{\textup{PDB}}\geq 1$ and (c)
$\overline{\gecost{B}} \geq 1$. From (a) and \eqref{eq:rgoal} we can derive
that (d) $v\geq 1$ for each goal variable $v\in\goal$.

We now iterate over all abstract states $\pdbstate$. If $\pdbstate$ is
\emph{not} an abstract goal state, we use (d) to derive from
\eqref{eq:rinefficientpdbstate} that $\overline{\rpdbstate{\pdbstate}}\geq 1$ and
consequently with \eqref{eq:rinefficientpdbstategeg}
$\overline{\rpdbstategeg{\pdbstate}} \geq 1$. If $\pdbstate$ is an abstract
goal state, we exploit that $d(\pdbstate)=0$ and use (c) with
\eqref{eq:rinefficientpdbstategeg} to derive $\overline{\rpdbstategeg{\pdbstate}}
\geq 1$. So for all abstract states $\pdbstate$ we derived $\overline{\rpdbstategeg{\pdbstate}}
\geq 1$. This gives with \eqref{eq:rinefficientPDB} that $\overline{r_{\textup{PDB}}}\geq 1$, a contradiction to (b).
\end{proof}

\pdbinductivitylemmaone*

\begin{proof}
Assume (a) $\rpdbstate{\pdbstate}\geq 1$, (b) $r_a\geq 1$, and (c)
$\overline{\rpdbprimedstate{\app{\pdbstate}{a^\alpha}}}\geq 1$.
From (a) and \eqref{eq:rinefficientpdbstate} we get (d) $v \geq 1$ for all $v\in \pdbstate$
and $\bar{v}\geq 1$ for all $v\in P\setminus\pdbstate$.
From (b) and \eqref{eq:actionvar}, we get
(e) $v'\geq 1$ for all $v\in\add(a)$, (f) $\overline{v'}\geq 1$ for all
$v\in\del(a)$, and (g) $\eqvar{v}\geq 1$ for all $v\in
\variables\setminus\effvars(a)$. From (d) and (g) and \eqref{eq:eqvars}, we derive for all $v\in
P\setminus\effvars(a)$ that $v'\geq 1$ if $v\in \pdbstate$ and
$\overline{v'}\geq 1$ if $v\notin\pdbstate$. Combining this with (e), (f) and
\eqref{eq:rinefficientpdbstate} from $\Hpdb'$ for $\app{\pdbstate}{a^\alpha}$, we derive
$\rpdbprimedstate{\app{\pdbstate}{a^\alpha}}\geq 1$, contradicting (c).
\end{proof}

\pdbinductivitylemmatwo*
\begin{proof}
  For each action $a\in \actions$ that is applicable in $\pdbstate$,
  we establish (a) $(\rpdbstategeg{\pdbstate}\land r_a)\rightarrow
  r'_{\textup{PDB}}$ with Lemma~\ref{lemma:inefficientPDB2}.

  Afterwards, the desired constraint can be derived by RUP, assuming (b)
  $\rpdbstategeg{\pdbstate}\geq 1$, (c) $\rtrans\geq 1$ and (d)
  $\overline{r'_{\textup{PDB}}}\geq 1$.

  For each action $a\in \actions$, we derive (e) $\overline{r_a}\geq 1$ as follows:
  If $\aalpha$ is applicable in $\pdbstate$, we use (a), (b) and (d).
  Otherwise, there is a $v\in \pre(a)\cap P$ with $v\notin \pdbstate$. From (b)
  and \eqref{eq:rinefficientpdbstategeg}, we derive $\rpdbstate{\pdbstate}\geq 1$, which
  with \eqref{eq:rinefficientpdbstate} gives us (f) $\overline{v}\geq 1$. Since
  $v\in\pre(a)$, we can use (f) with \eqref{eq:actionvar} to derive (e) also
  for this case.
  Since we have derived (e) for all actions $a\in\actions$, we can use
  \eqref{eq:rtrans} to derive $\overline{\rtrans}\geq 1$, a contradiction to
  (c).
\end{proof}

\pdbinductivitylemma*

\begin{proof}
  We first use lemma \ref{lemma:inefficientPDB3} to establish for each abstract state
  $\pdbstate$ (a) $(\rpdbstategeg{\pdbstate}\land \rtrans) \rightarrow
  r'_{\textup{PDB}}$.

  The rest can be derived by RUP, assuming (b) $r_{\textup{PDB}}\geq 1$, (c)
  $\rtrans\geq 1$, and (d) $\overline{r'_{\textup{PDB}}}\geq 1$.
  Using (c) and (d) with (a), we can derive for each abstract state $s_\alpha$
  that $\overline{\rpdbstategeg{\pdbstate}}\geq 1$. With \eqref{eq:rinefficientPDB}, this
  gives $\overline{r_{\textup{PDB}}}\geq 1$, a contradiction to (b).
\end{proof}

\section{Proof-Logging \hmax: Proofs}

\hmaxstatelemma*

\begin{proof}
Assume (a) $r_s\geq 1$, (b) $\gecost{\max\{0,B-\hmax(s)\}}\geq 1$, and (c)
$\overline{\rmax_s}\geq 1$.
From (a), we get for all $v\in \variables\setminus s$
that $\bar{v}\geq 1$, so with \eqref{eq:rmaxv} we
  get for all such $v$ that $r_{v,s}\geq 1$ (*). For all $v\in s$, the heuristic
determines $\boundedvmax(v)=\vmax(v)=0$,
and thus $B-\hmax(s)+\boundedvmax(v)\leq B$,
so from (b) and \eqref{eq:rmaxv}
we get $r_{v,s}\geq 1$ also for these variables (**).

Assumption (b) corresponds to $\mmgecost{B-\hmax(s)}$.

Together with (*) and (**), we derive with \eqref{eq:rmaxs} that $\rmax_s\geq
1$, contradicting (c).
\end{proof}

\hmaxgoallemma*

\begin{proof}
Assume (a) $\rgoal\geq 1$, (b) $\overline{\gecost{B}}\geq 1$ and (c)
$\rmax_s\geq 1$.

If $G\neq\emptyset$, consider a goal variable $g$ with maximum
max value. For this $g$, it holds that $\vmax(g)=\hmax(s)$. From (a) we
derive that $g\geq 1$, which gives together with (b) and \eqref{eq:rmaxv}
$\overline{r_g}\geq 1$. Thus we get from \eqref{eq:rmaxs} that
$\overline{\rmax_s}\geq 1$, contradicting (c).

If $G=\emptyset$ then $\hmax(s)=0$ for all states $s$ and $\boundedvmax(v)
= 0$ for all variables $v$. From (c) we get that $\mmgecost{B}\geq 1$,
contradicting (b).
\end{proof}

\hmaxinductivityaux*

\begin{proof}
  The overall argument of the proof is as follows:

  If the action has a precondition $p$, we can conclude that $p\geq 1$, which
  allows us to derive with \eqref{eq:rmaxv} that
  $\mmgecost{B-\hmax(s)+\boundedvmax(p)}\geq 1$. We can establish
  $\mmgeprimedcost{B-\hmax(s)}$ using lemmas \ref{lemma:costs1-variant} and
  \ref{lemma:costs3}.

  To establish $r'_{v,s}\geq 1$ for all variables, we consider the precondition
  with maximal $\boundedvmax$ value.

  If $\boundedvmax(p) = \hmax(s)$, then
  $\mmgecost{B-\hmax(s)+\boundedvmax(p)}\geq 1$ corresponds to
  $\mmgecost{B}\geq 1$. We derive $\mmgeprimedcost{B}\geq 1$, which is
  sufficient to establish $r'_{v,s}\geq 1$. 

  If $\boundedvmax(p) < \hmax(s)$, we
  establish $r'_{v,s}\geq 1$ depending on whether and how $v$ occurs in the
  effect of $a$. If it is a delete effect, then $\overline{v'}\geq 1$. If it is an
  add effect, we exploit that $p$ has maximal $\boundedvmax$ value among the
  preconditions, so from the definition of the maximum heuristic we can see
  that $\boundedvmax(v)\leq \boundedvmax(p) +\cost(a)$. Establishing
  $\mmgeprimedcost{B-\hmax(s)+\boundedvmax(p)+\cost(a)}$ and using this insight
  gives $\mmgeprimedcost{B-\hmax(s)+\boundedvmax(v)}\geq 1$. If
  $v\notin\effvars(a)$, we can carry over the reason for $r_{v,s}\geq 1$ to the
  primed variables.
  
  If the precondition of the action is empty, we build on
  $\mmgecost{B-\hmax(s)}$, which gives $\mmgeprimedcost{B-\hmax(s)+\cost(a)}$
  and $\mmgeprimedcost{B-\hmax(s)}$ using lemmas \ref{lemma:costs1-variant} and
  \ref{lemma:costs3}. For establishing $r'_{v,s}\geq 1$ for all variables, we
  proceed as before for case $\boundedvmax(p) < \hmax(s)$ (replacing
  $\boundedvmax(p)$ with $0$). For the special case $\hmax(s)=0$,
  we can instead use the shorter argument as in the previous case
  $\boundedvmax(p) = \hmax(s)$, building on $\mmgecost{B}\geq 1$.

  \bigskip

  We not turn to specifying the details of this overall argument:
 
  If $\pre(a)\neq\emptyset$, let $p\in\pre(a)$ be a precondition with maximal $\boundedvmax$ value.
  
  If $\boundedvmax(p) = \hmax(s)$, we first derive for all variables $v$ that
  $(\rmax_s \land r_a)\rightarrow r'_{v,s}$: We use Lemma~\ref{lemma:costs3}
  to establish (*) $(\mmgecost{B}\land
  \deltacost{\cost(a)})\rightarrow\mmgeprimedcost{B}$ and proceed by RUP,
  assuming (a) $\rmax_s\geq 1$ and (b) $r_a\geq 1$. From (a) and
  \eqref{eq:rmaxs} we get (c) $r_{p,s}\geq 1$.  From (b) and \eqref{eq:actionvar}
  we get (d) $p\geq 1$, (e) $\overline{\geprimedcost{B}}\geq 1$ and (f)
  $\deltacost{\cost(a)}\geq 1$. From (c) and (d), we get with
  $\eqref{eq:rmaxv}$ that (g) $\mmgecost{B-\hmax(s)+\boundedvmax(p)}\geq 1$. Then
  (f), (g) and (*) yield $\mmgeprimedcost{B} \geq 1$, a contradiction to (e).
  \smallskip

  If $\boundedvmax(p) < \hmax(s)$, we use Lemma~\ref{lemma:costs3} to establish
  (**) $(\mmgecost{B-\hmax(s)+\boundedvmax(v)}\land
  \deltacost{\cost(a)})\rightarrow\mmgeprimedcost{B-\hmax(s)+\boundedvmax(v)+cost(a)}$
  for all $v\in\variables$.
  We use Lemma~\ref{lemma:costs1-variant} to establish
  (***) $\mmgeprimedcost{B-\hmax(s) + \boundedvmax(v) +\cost(a)}\rightarrow
  \mmgeprimedcost{B-\hmax(s) + \boundedvmax(v)}$ for every
  $v\in\variables\setminus\effvars(a)$.
  For every $v\in\add(a)$, it holds by the definition of $\hmax$ that 
  $\vmax(v)\leq \vmax(p)+\cost(a)$ and
  consequently $\boundedvmax(v)\leq \boundedvmax(p) +\cost(a)$. We thus can
  establish with Lemma~\ref{lemma:costs1-variant}
  that (****) $\mmgeprimedcost{B-\hmax(s) + \boundedvmax(p) +\cost(a)}\rightarrow
  \mmgeprimedcost{B-\hmax(s) + \boundedvmax(v)}$.
  \smallskip

  Then we derive $(\rmax_s \land r_a)\rightarrow r'_{v,s}$ for all variables $v$ by
  RUP: Assume (h) $\rmax_s\geq 1$, (i) $r_a\geq 1$, and (j)
  $\overline{r'_{v,s}}\geq 1$.
  From (i) and \eqref{eq:actionvar} we get (k) $\deltacost{\cost(a)}\geq 1$. From (h) with \eqref{eq:rmaxv} and \eqref{eq:rmaxs} we derive (l)
  $\mmgecost{B-\hmax(s)+\boundedvmax(p)}\geq 1$ as in case $\boundedvmax(p)
  = \hmax(s)$.

  If $v\in\del(a)$, we have from (i) and \eqref{eq:actionvar} that
  $\overline{v'}\geq 1$. With \eqref{eq:rmaxv} this yields $r'_{v,s}\geq 1$,
  contradicting (j).

  If $v\in\add(a)$, we use (**) with (l) and (k) to derive
  $\mmgeprimedcost{B-\hmax(s)+\boundedvmax(p)+\cost(a)}$,
  which together with (****) gives
  $\mmgeprimedcost{B-\hmax(s)+\boundedvmax(v)}$. With \eqref{eq:rmaxv} from
  $\Hmax'$, we get $r'_{v,s}\geq 1$, contradicting (j).

  If $v\in\variables\setminus\effvars(a)$, (j) and \eqref{eq:rmaxv} yield (m) $v'\geq 1$ and (n)
  $\overline{\mmgeprimedcost{B-\hmax(s) + \boundedvmax(v)}}\geq 1$.
  From (i) and \eqref{eq:actionvar} we get $\eqvar{v}\geq 1$, which with (m)
  yields (o) $v\geq 1$. We use (n) with (***) to get
  $\overline{\mmgeprimedcost{B-\hmax(s)+\boundedvmax(v)+cost(a)}}\geq 1$,
  which gives with (**) and (k) that $\overline{\mmgecost{
  B-\hmax(s)+\boundedvmax(v)}}\geq 1$. 
  With \eqref{eq:rmaxv} and (o) this yields $\overline{r_{v,s}}\geq 1$,
  which with \eqref{eq:rmaxs} gives $\overline{\rmax_s}\geq 1$, contradicting
  (h).
  \smallskip

  Continuing with the case $\boundedvmax(p) = \hmax(s)$, we now derive
  $(\rmax_s \land r_a)\rightarrow \mmgeprimedcost{B-\hmax(s)}$: 
  We use Lemma~\ref{lemma:costs3}
  to establish ($\dagger$) $(\mmgecost{B-\hmax(s) +\boundedvmax(p)}\land
  \deltacost{\cost(a)})\rightarrow\mmgeprimedcost{B-\hmax(s)
  +\boundedvmax(p)+\cost(a)}$. With Lemma~\ref{lemma:costs1-variant} on the
  primed variables, we get ($\dagger\dagger$) $\mmgeprimedcost{B-\hmax(s)
  +\boundedvmax(v)+\cost(a)}\rightarrow \mmgeprimedcost{B-\hmax(s)}$.
  We then proceed by RUP, assuming (p) $\rmax_s\geq 1$, (q) $r_a\geq 1$ and (r)
  $\overline{\mmgeprimedcost{B-\hmax(s)}}\geq 1$.
  From (p) and \eqref{eq:rmaxs} we get (s) $r_{p,s}\geq 1$.  From (q) and
  \eqref{eq:actionvar} we get (t) $p\geq 1$ and (u) $\deltacost{\cost(a)}\geq
  1$. We use (s) and (t) and \eqref{eq:rmaxv} to obtain $\mmgecost{B-\hmax(s)+\boundedvmax(p)}\geq
  1$, which with (u), ($\dagger$) and ($\dagger\dagger$) propagates to
  $\mmgeprimedcost{B-\hmax(s)}\geq 1$, a contradiction to (r).
  \smallskip

  RUP can now derive that $(\rmax_s \land r_a)\rightarrow {\rmax_s}'$: Assume
  (v) $\rmax_s\geq 1$, (w) $r_a\geq 1$ and (x) $\overline{{\rmax_s}'}\geq 1$.
  For all $v\in\variables$, we get with (v), (x) and the previously derived
  $(\rmax_s \land r_a)\rightarrow r'_{v,s}$, that $r'_{v,s}\geq 1$. We use (v),
  (x) and the previously derived $(\rmax_s \land r_a)\rightarrow\mmgeprimedcost{B-\hmax(s)}$ to get
  $\mmgeprimedcost{B-\hmax(s)}\geq 1$.  With \eqref{eq:rmaxs} from $\Hmax'$,
  these yield ${\rmax_s}'\geq 1$, contradicting (x).
  \smallskip

  If $\pre(a) = \emptyset$ the proof is analogous but simpler: instead of using
  $p$ to derive $\gecost{B-\hmax(s)+\boundedvmax(p)}\geq 1$, we use
  $\gecost{B-\hmax(s)}\geq 1$. Then case $\hmax(s)=0$ works like the previous
  case for $\boundedvmax(p) = \hmax(s)$. Otherwise, we proceed as in the case
  $\boundedvmax(p) < \hmax(s)$. If $v\in\add(a)$, we used $\boundedvmax(p)$ to
  refer to $\hmax(s,\pre(a))$, which is $0$ with an empty precondition. We can
  adapt the proof by replacing all occurrences of $\boundedvmax(p)$ with $0$.
\end{proof}

\hmaxinductivitylemma*
\begin{proof}
  Establish with Lemma~\ref{lemma:hmaxaction} for every $a\in\actions$ that
  (a) $(\rmax_s \land r_a)\rightarrow {\rmax_s}'$.
  Afterwards, continue by RUP: Assume (b) $\rmax_s\geq 1$, (c) $\rtrans\geq 1$,
  and (d) $\overline{{\rmax_s}'}\geq 1$. With (a), (b) and (d), derive (e)
  $\overline{r_a}\geq 1$ for every $a\in\actions$. With \eqref{eq:rtrans}, this
  gives $\overline{\rtrans}\geq 1$, contradicting (c).
\end{proof}

\section{State Set Extension Lemma}

The proofs in the next section require to switch from a description of a set of
states in form of a constraint that fixes the value of some state variables, to
a description of the same set of states in form of an enumeration of tighter
constraints describing subsets of the original set of states. The next lemma
shows how to construct a proof for this statement.

In the following, we write $x^b$ to denote the literal over $x$ that evaluates to $1$ exactly when $x \mapsto b$,
i.e., $x^b=x$ if $b=1$ and $x^b=\overline{x}$ if $b=0$.

\begin{lemma}
    Let $Y,Z$ be two sets of variables with $Y\subseteq Z$.
    Let $\alpha$ be a fixed assignment of the variables in $Y$ and let $\mathcal{B}$ be the set of all assignments $\beta$ of $Z$ with $\beta\supseteq\alpha$ extending $\alpha$.
    Suppose that there are reified constraints
    \begin{equation}
        \label{eq:eqalpha}
        r_\alpha \Rightarrow \sum_{y\in Y} y^{\alpha(y)}\geq |Y|
    \end{equation}
    and for each $\beta\in\mathcal{B}$ constraints
    \begin{equation}
        \label{eq:eqbeta}
        r_\beta \Leftarrow \sum_{z\in Z} z^{\beta(z)}\geq |Z|
    \end{equation}
    defining variables for all these assignments.
    Then there is a CPR derivation from \eqref{eq:eqalpha} and \eqref{eq:eqbeta} of
    \begin{equation}
        \label{eq:partialstatesetextension}
        \overline{r_\alpha} + \sum_{\beta\in\mathcal{B}}r_\beta \geq 1
    \end{equation}
    in $\mathcal{O}(|\mathcal{B}|)$ steps.
\label{lemma:partialstatesetextension}
\end{lemma}

\begin{proof}
    Let $Z\setminus Y = \{z_1,\dots,z_n\}$ and $Z_i=\{z_1,\dots,z_i\}\cup Y$.
    Additionally, let $\mathcal{B}_{i}$ be the set of all assignments of $Z_i$ with $\beta_i\supseteq\alpha$ extending $\alpha$
    (i.e. $\mathcal{B}_n=\mathcal{B}$ and $\mathcal{B}_0=\{\alpha\}$).

    Step $1$:
    For each of the $2^{n-1}$ many assignments
    $\beta_{n-1}: Z_{n-1}\rightarrow\{0,1\}$ in ${\mathcal{B}}_{n-1}$ we derive the constraint
    \begin{equation*}\tag{a}
        \sum_{\beta\supseteq\beta_{n-1}}r_\beta + \sum_{z\in Z_{n-1}} z^{1-\beta_{n-1}(z)} \geq 1
    \end{equation*}
    by RUP:
    Assuming the negation of (a), we get
    (b) $\overline{r_\beta} \geq 1$ for all $\beta\supseteq\beta_{n-1}$, and
    (c) $z^{\beta_{n-1}(z)} \geq 1$ for all $z\in Z_{n-1}$.
    Note that $\{\beta_{n-1}\cup\{z_n\mapsto0\},\beta_{n-1}\cup\{z_n\mapsto1\}\}\subseteq \mathcal{B}_n$.
    Thus, by using (b) and (c) with \eqref{eq:eqbeta} for $\beta_{n-1} \cup \{z_n \mapsto 0\}$, we derive
    (d) $z_n\geq 1$, and similarly by using (b) and (c) with \eqref{eq:eqbeta} for
    $\beta_{n-1} \cup \{z_n \mapsto 1\}$ we derive $\overline{z_n}\geq 1$,
    a contradiction to (d).

    Step $k$:
    As induction hypothesis we have
    \begin{equation*}\tag{e}
        \sum_{\beta\supseteq\beta_{n-k+1}}r_\beta + \sum_{z\in Z_{n-k+1}} z^{1-\beta_{n-k+1}(z)} \geq 1
    \end{equation*}
    for each $\beta_{n-k+1}\in{\mathcal{B}}_{n-k+1}$.
    We derive the constraint
    \begin{equation*}\tag{f}
        \sum_{\beta\supseteq\beta_{n-k}}r_\beta + \sum_{z\in Z_{n-k}} z^{1-\beta_{n-k}(z)} \geq 1
    \end{equation*}
    by RUP:
    Assuming the negation of (f), we get
    (g) $\overline{r_\beta} \geq 1$ for all $\beta\supseteq\beta_{n-k}$, and
    (h) $z^{\beta_{n-k}(z)} \geq 1$ for all $z\in Z_{n-k}$.
    We use (g) and (h) with (e) for $\beta_{n-k} \cup \{z_{n-k+1} \mapsto 0\}$
    obtained in step $k-1$,
    to derive
    (i) $z_{n-k+1}\geq1$, and similarly we use (g) and (h) with (e) for
    $\beta_{n-k} \cup \{z_{n-k+1} \mapsto 1\}$ to derive
    $\overline{z_{n-k+1}}\geq1$, a contradiction to (i).

    The constraint derived in step $n$ is
    \begin{equation*}\tag{j}
        \sum_{\beta\in\mathcal{B}} r_\beta + \sum_{y\in Y} y^{1-\alpha(y)} \geq 1
    \end{equation*}
    as $\beta_0=\alpha$ and $Z_0=Y$.

    We can now derive the desired constraint \eqref{eq:partialstatesetextension} by RUP:
    Assuming that $r_\alpha + \sum_{\beta\in\mathcal{B}}\overline{r_\beta} \geq 1+|\mathcal{B}|$
    we obtain
    (k) $r_\alpha\geq1$ and
    (l) $\overline{r_\beta} \geq 1$ for all $\beta\in\mathcal{B}$.
    From (k) and \eqref{eq:eqalpha} we obtain that
    (m) $y^{\alpha(y)} \geq 1$ for all $y\in Y$,
    and thus (j) with (l) and (m) leads to a contradiction.
\end{proof}

\section{Efficiently Proof-Logging Pattern Database Heuristics}

In the paper, we described PDB certificates where the overhead is not within
a constant factor of an efficient implementation of a PDB heuristic, which only
traverses the part of the abstract state space that is backwards-reachable from
the goal. Here we describe an alternative approach that is constant-factor
efficient for such an implementation.

\begin{definition}
Let $\Pi = \langle \variables, \actions, \init, \goal\rangle$ be a STRIPS
planning task and $P\subseteq \variables$. The \emph{PB circuit for the PDB
heuristic with pattern $P$} is the PB circuit $\langle\Hpdb,r_{\textup{PDB}}\rangle$ with input
variables $\variables\cup\cvars$, where the set represented by $\Hpdb$ contains for each
  abstract state $\pdbstate\subseteq P$ with $d(\pdbstate)<\infty$ two
  constraints
\begin{align}
 \rpdbstate{\pdbstate}&\Leftrightarrow \sum\nolimits_{v\in \pdbstate}v
+  \sum\nolimits_{v\in P\setminus\pdbstate}\bar{v}\geq
  |P|\text{, and}\label{eq:rpdbstate-eff}\\
 \rpdbstategeg{\pdbstate}&\Leftrightarrow \rpdbstate{\pdbstate}
+ \mmgecost{B-d({\pdbstate})}\geq 2,\label{eq:rpdbstategeg-eff}
  \intertext{and with $S=\{\pdbstate\subseteq P\mid
  d(\pdbstate) < \infty\}$ the constraint}
  \rpdbinfinite &\Leftrightarrow \sum\nolimits_{\pdbstate\in S}
  \overline{\rpdbstate{\pdbstate}} \geq |S|\label{eq:rpdbinfinite-eff}\\
  \intertext{as well as the final reification}
  r_{\textup{PDB}} &\Leftrightarrow \rpdbinfinite + \sum\nolimits_{\pdbstate\in S}
  \rpdbstategeg{\pdbstate} \geq 1.
\label{eq:rPDB-eff}
\end{align}
\end{definition}

On the evaluation of a state $s$, the heuristic always returns reification variable
$r_{\textup{PDB}}$ to the search. In the following we discuss how the required
proofs can be generated.

We begin with the state lemma, which requires a new proof for each evaluated
state.

\begin{lemma}
RUP can derive the state lemma\\
$(r_s\land \gecost{\max\{0,B-h(s)\}})
    \rightarrow r_{\textup{PDB}}$ from $\cinit \cup \Hpdb\cup \mathcal C_s$,\\
where $\mathcal C_s = \{r_s
\Leftrightarrow \sum_{v\in s} v + \sum_{v\in\variables\setminus s}\bar v \geq
|\variables|\}$.
\label{lemma:PDBstatelemma-eff}
\end{lemma}

\begin{proof}
Assume (a) $r_s\geq 1$, (b) $\gecost{\max\{0,B-h(s)\}}\geq 1$, and (c)
$\overline{r_{\textup{PDB}}}\geq 1$.  From
(a) and $\mathcal C_s$, we receive $v\geq 1$ for each $v\in s$ and $\bar v\geq 1$ for each
$v\in\variables\setminus s$,
and in particular (d) $v\geq 1$ for each $v\in\alpha(s)$ and $\bar v\geq 1$ for each
$v\in P\setminus \alpha(s)$. 

If $d(\alpha(s))<\infty$, we derive with constraints \eqref{eq:rpdbstate-eff} and
(d), that (e) $\rpdbstate{\alpha(s)}\geq 1$. Since $h(s) = d(\alpha(s))$, we
can use (b) and (e) with \eqref{eq:rpdbstategeg-eff} to derive (f)
$\rpdbstategeg{\alpha(s)}$. With \eqref{eq:rPDB-eff}, this gives us
$r_{\textup{PDB}}$, a contradiction to (c).

If $d(\alpha(s))=\infty$, we get from constraints \eqref{eq:rpdbstate-eff} and
(d) for every $\pdbstate$ with $d(\pdbstate)<\infty$, that (g)
$\overline{\rpdbstate{\pdbstate}}\geq 1$, which gives with
\eqref{eq:rpdbinfinite-eff} that $\rpdbinfinite\geq 1$. With \eqref{eq:rPDB-eff}, this gives us
$r_{\textup{PDB}}\geq 1$, a contradiction to (c).
\end{proof}

Since the PB circuit for the heuristic is the same for all evaluated states, we
only need to include the proof for the goal lemma and for the inductivity lemma
once in the overall generated proof. For the goal lemma, we use the following:

\begin{lemma}
  It is possible to derive the goal lemma\\
$(\rgoal\land{r_{\textup{PDB}}})\rightarrow  {\gecost{B}}$ from $\cgoal\cup \Hpdb$.
\label{lemma:PDBgoallemma-eff}
\end{lemma}

\begin{proof}
Define a reification of the abstract goal as
  \[r_{G^\alpha}\Leftrightarrow \sum_{v\in G\cap P} v \geq |G\cap P|.\]
We first establish
  (a) $\overline{r_{G^\alpha}} + \sum_{\pdbstate\subseteq P,G\cap P\subseteq
  \pdbstate} \rpdbstate{\pdbstate} \geq 1$
  using lemma \ref{lemma:partialstatesetextension}: Use $\alpha = \{v\mapsto
  1\mid v\in G\cap P\}$ and for each abstract state $\pdbstate$ with $G\cap
  P\subseteq  \pdbstate$, $\mathcal B$ contains the assignment
  $\beta:P\rightarrow\{0,1\}$ with $\beta(v) = 1$ iff $v\in \pdbstate$.

Next, we use RUP to obtain
(b) $r_G\rightarrow  \overline{\rpdbinfinite}$, assuming
(c) $r_G \geq 1$ and (d) $\rpdbinfinite \geq 1$.
From \eqref{eq:rpdbinfinite-eff} and (d) we obtain
(e) $\overline{\rpdbstate{\pdbstate}} \geq 1$ for each
$\rpdbstate{\pdbstate} \in S$ and in particular for each
$\rpdbstate{\pdbstate}$ where $\pdbstate$ is an abstract goal state which have
$d(\pdbstate) = 0 < \infty$
From (c), we get $v\geq 1$ for all $v\in G$, an in particular for $v\in G\cap
P$. From these we obtain (f) $r_{G^\alpha}\geq 1$
From (a) and (e) we obtain $\overline{r_{G^\alpha}}$ contradicting (f).
\smallskip

We now can establish the desired constraint via RUP:
Assume
(g) $\rgoal\geq 1$,
(h) $r_{\textup{PDB}}\geq 1$ and
(i) $\overline{\gecost{B}} \geq 1$. From (g) and \eqref{eq:rgoal} we can derive
that
(j) $v\geq 1$ for each goal variable $v\in\goal$.

We now iterate over all abstract states $\pdbstate$ with $d(\pdbstate)
<\infty$. If $\pdbstate$ is
\emph{not} an abstract goal state, we use (j) to derive from
\eqref{eq:rpdbstate-eff} that $\overline{\rpdbstate{\pdbstate}}\geq 1$ and
consequently with \eqref{eq:rpdbstategeg-eff}
$\overline{\rpdbstategeg{\pdbstate}} \geq 1$. If $\pdbstate$ is an abstract
goal state, we exploit that $d(\pdbstate)=0$ and use (i) with
\eqref{eq:rpdbstategeg-eff} to derive $\overline{\rpdbstategeg{\pdbstate}}
\geq 1$. So for all abstract states $\pdbstate$ with $d(\pdbstate) <\infty$ we
derived
(k) $\overline{\rpdbstategeg{\pdbstate}} \geq 1$. With (h), (k) and
\eqref{eq:rPDB-eff}, we derive $\rpdbinfinite\geq 1$, and from (g) and (b) that
$\overline{\rpdbinfinite}\geq 1$, a contradiction.
\end{proof}
  
For the inductivity lemma, we develop the derivation by means of several lemmas.
The first one considers the case where applying an induced abstract action
$\aalpha$ in abstract state $\pdbstate$ leads to an abstract successor state
that is backwards reachable from the goal.

\begin{lemma}
For each action $a\in \actions$ and abstract state $\pdbstate$ such that $\aalpha$
is applicable in $\pdbstate$ and $d(\app{\pdbstate}{a^\alpha})<\infty$, RUP can derive\\
$(\rpdbstate{\pdbstate}\land r_a)\rightarrow
\rpdbprimedstate{\app{\pdbstate}{a^\alpha}}$ from $\ctrans\cup \Hpdb\cup \Hpdb'$.
\label{lemma:PDB1-eff}
\end{lemma}

\begin{proof}
By RUP assume (a) $\rpdbstate{\pdbstate}\geq 1$, (b) $r_a\geq 1$, and (c)
$\overline{\rpdbprimedstate{\app{\pdbstate}{a^\alpha}}}\geq 1$.
From (a) and \eqref{eq:rpdbstate-eff} we get (d) $v \geq 1$ for all $v\in \pdbstate$
and $\bar{v}\geq 1$ for all $v\in P\setminus\pdbstate$.
From (b) and \eqref{eq:actionvar}, we get
(e) $v'\geq 1$ for all $v\in\add(a)$, (f) $\overline{v'}\geq 1$ for all
$v\in\del(a)$, and (g) $\eqvar{v}\geq 1$ for all $v\in
\variables\setminus\effvars(a)$. From (d) and (g) and \eqref{eq:eqvars}, we derive for all $v\in
P\setminus\effvars(a)$ that $v'\geq 1$ if $v\in \pdbstate$ and
$\overline{v'}\geq 1$ if $v\notin\pdbstate$. Combining this with (e), (f) and
the primed \eqref{eq:rpdbstate-eff} from $\Hpdb'$ for $\app{\pdbstate}{a^\alpha}$, we derive
$\rpdbprimedstate{\app{\pdbstate}{a^\alpha}}\geq 1$, contradicting (c).
\end{proof}

The next lemma again considers individual actions and abstract states, but
takes cost into account and allows to derive that the successor state-cost pair is
contained in the invariant of the certificate.
\begin{lemma}
For each action $a\in \actions$ and abstract state $\pdbstate$ such that $\aalpha$
  is applicable in $\pdbstate$ and $d(\app{\pdbstate}{a^\alpha})<\infty$, it is possible to derive\\
$(\rpdbstategeg{\pdbstate}\land r_a)\rightarrow
  r'_{\textup{PDB}}$ from $\ctrans\cup \Hpdb\cup \Hpdb'$.
\label{lemma:PDB2-eff}
\end{lemma}

\begin{proof}
We start by deriving some constraints over costs that we will use later in a RUP
proof.

First, we derive (a) $(\gecost{\max\{0,
B-d(\pdbstate)\}}\land\deltacost{\cost(a)})\rightarrow\geprimedcost{\min\{B,\max\{0,B-d(\pdbstate)+\cost(a)\}\}}$
as described in lemma \ref{lemma:costs3}.

We know that for the abstract goal distance it holds that
$d(\app{\pdbstate}{a^\alpha})+\cost(a^\alpha)\geq d(\pdbstate)$.
Together with $\cost(a^\alpha) =\cost(a)$, it follows that
$B-d(\app{\pdbstate}{a^\alpha})\leq B-d(\pdbstate) + \cost(a)$.
We use this to derive (b)
$\mmgeprimedcost{B-d(\pdbstate)+\cost(a)}
\rightarrow\mmgeprimedcost{B-d(\app{\pdbstate}{a^\alpha})}$ as described
in lemma \ref{lemma:costs1-variant}.

In addition, we establish (c) $(\rpdbstate{\pdbstate}\land r_a)\rightarrow
\rpdbprimedstate{\app{\pdbstate}{a^\alpha}}$ by RUP (lemma \ref{lemma:PDB1-eff}).

Now we can derive the constraint in the claim by RUP, assuming that
(d) $\rpdbstategeg{\pdbstate}\geq 1$, (e) $r_a\geq 1$, and
(f) $\overline{r'_{\textup{PDB}}}\geq 1$.
From (d), we get with \eqref{eq:rpdbstategeg-eff} that (g) $\rpdbstate{\pdbstate}\geq 1$
and (h) $\mmgecost{B-d({\pdbstate})}\geq 1$. From (e), (g) and (c) we
derive (i) $\rpdbprimedstate{\app{\pdbstate}{a^\alpha}}\geq 1$.

From (e) and \eqref{eq:actionvar}, we derive (j) $\deltacost{\cost(a)}\geq 1$.
From (a) together with (h) and (j) we derive that (k)
$\mmgeprimedcost{B-d(\pdbstate)+\cost(a)}\geq 1$.
From (b) together with (k) we derive that (l)
$\mmgeprimedcost{B-d(\app{\pdbstate}{a^\alpha})}\geq1$.
From (l), (i) and the primed constraint \eqref{eq:rpdbstategeg-eff} from
$\Hpdb'$ we get
$\rpdbprimedstategeg{\app{\pdbstate}{a^\alpha}}$. With the primed version of
\eqref{eq:rPDB-eff} from $\Hpdb'$ this yields $r'_{\textup{PDB}}\geq 1$,
contradicting (e).
\end{proof}

The next lemma considers the case that the successor state of some
backwards-reachable state is not backwards reachable, and allows to derive that
this state-cost pair is contained in the invariant of the certificate.
\begin{lemma}
For each action $a\in \actions$ and abstract state $\pdbstate$ such that
$\aalpha$ is applicable in $\pdbstate$ and $d(\pdbstate)<\infty$ and
$d(\app{\pdbstate}{a^\alpha})=\infty$, it is possible to derive\\
$(\rpdbstategeg{\pdbstate}\land r_a)\rightarrow
r'_{\textup{PDB}}$ from $\ctrans\cup \Hpdb\cup \Hpdb'$.
\label{lemma:PDB2inf-eff}
\end{lemma}

\begin{proof}
By RUP assume
(a) $\rpdbstategeg{\pdbstate} \geq 1$ and
(b) $r_a \geq 1$ and
(c) $\overline{r'_{\textup{PDB}}} \geq 1$.
From (a) and \eqref{eq:rpdbstategeg-eff} we derive
(d) $\rpdbstate{\pdbstate} \geq 1$.

From (b) and \eqref{eq:actionvar}, we get
(e) $v'\geq 1$ for all $v\in\add(a)$,
(f) $\overline{v'}\geq 1$ for all $v\in\del(a)$, and
(g) $\eqvar{v}\geq 1$ for all $v\in\variables\setminus\effvars(a)$. From (d)
and \eqref{eq:rpdbstate-eff} we obtain that
(h) $v\geq 1$ for all $v\in\pdbstate$ and
(i) $\overline{v}\geq1$ for all $v\in P\setminus\pdbstate$.
Using \eqref{eq:eqvars} with (g), (h) and (i), we derive for all $v\in P\setminus\effvars(a)$ that
(j) $v'\geq 1$ if $v\in \pdbstate$ and
(k) $\overline{v'}\geq 1$ if $v\notin\pdbstate$.
As $d(\app{\pdbstate}{a^\alpha})=\infty$,
we can derive from the primed versions
of \eqref{eq:rpdbstate-eff} from $\Hpdb'$ and (e), (f), (j) and (k) that
(l) $\overline{\widehat{{\rpdbstate{\pdbstate}}'}} \geq 1$
for all $\widehat{\rpdbstate{\pdbstate}} \in S$. Therefore, using the primed
version of \eqref{eq:rpdbinfinite-eff} from $\Hpdb'$ we obtain
(m) ${\rpdbinfinite}' \geq 1$,
which with \eqref{eq:rPDB-eff} yields $r'_{\textup{PDB}} \geq 1$ contradicting (b).
\end{proof}

The next lemma generalizes the previous lemmas from individual actions to the
entire transition relation:

\begin{lemma}
  For each abstract state $\pdbstate$ with $d(\pdbstate)<\infty$, it is possible to derive\\
$(\rpdbstategeg{\pdbstate}\land \rtrans)
  \rightarrow r'_{\textup{PDB}}$ from $\ctrans\cup \Hpdb\cup \Hpdb'$.
\label{lemma:PDB3-eff}
\end{lemma}

\begin{proof}
For each action $a\in \actions$ that is applicable in $\pdbstate$, we establish
(a) $(\rpdbstategeg{\pdbstate}\land r_a)\rightarrow r'_{\textup{PDB}}$ by means
of lemmas \ref{lemma:PDB2-eff} and \ref{lemma:PDB2inf-eff}.

  Afterwards, the desired constraint can be derived by RUP, assuming (b)
  $\rpdbstategeg{\pdbstate}\geq 1$, (c) $\rtrans\geq 1$ and (d)
  $\overline{r'_{\textup{PDB}}}\geq 1$.

  For each action $a\in \actions$, we derive (e) $\overline{r_a}\geq 1$ as follows:
  If $\aalpha$ is applicable in $\pdbstate$, we use (a), (b) and (d).
  Otherwise, there is a $v\in \pre(a)\cap P$ with $v\notin \pdbstate$. From (b)
  and \eqref{eq:rpdbstategeg-eff}, we derive $\rpdbstate{\pdbstate}\geq 1$, which
  with \eqref{eq:rpdbstate-eff} gives us (f) $\overline{v}\geq 1$. Since
  $v\in\pre(a)$, we can use (f) with \eqref{eq:actionvar} to derive (e) also
  for this case.
  Since we have derived (e) for all actions $a\in\actions$, we can use
  \eqref{eq:rtrans} to derive $\overline{\rtrans}\geq 1$, a contradiction to
  (c).
\end{proof}

The previous lemmas have covered the cases where the considered state before a
transition is backwards-reachable. The next series of lemmas covers the cases
where a transition starts in a region that is not backwards-reachable.

We say that action $a$ is \emph{consistent} with abstract state $\pdbstate$, if
there is an abstract state $\widehat{\pdbstate}$ such that the application of
$\aalpha$ in $\widehat{\pdbstate}$ leads to $\pdbstate$. An action $a$ is
consistent with $\pdbstate$ iff
\begin{itemize}
  \item all add effects of $a$ from $P$ are true in $\pdbstate$\\ $\add(a)\cap P\subseteq \pdbstate$,
  \item all delete effects of $a$ from $P$ are false in $\pdbstate$\\ $\del(a)\cap P\cap\pdbstate = \emptyset$, and
  \item all preconditions of $a$ from $P$ that are not affected by $a$ are
    still true in $\pdbstate$: $(\pre(a)\cap P)\setminus \effvars(a)\subseteq
    \pdbstate$
\end{itemize}
If the action is not consistent with $\pdbstate$, we say it is \emph{inconsistent}.

The following lemma shows that we can derive that a backwards-reachable state
cannot be reached with an inconsistent action from a region that is not
backwards-reachable.

\begin{lemma}
  For each action $a\in \actions$ and abstract state $\pdbstate$ with
  $d(\pdbstate)<\infty$ such that $a$ is inconsistent with $\pdbstate$, RUP can
  derive\\
  $({\rpdbinfinite\land r_a)\rightarrow
  \overline{\rpdbstate{\pdbstate}}'}$ from $\ctrans\cup \Hpdb\cup \Hpdb'$.
  \label{lemma:PDBinftyinconsistentaction}
\end{lemma}

\begin{proof}
  Assume (a) ${\rpdbstate{\pdbstate}}'\geq 1$, (b) $r_a\geq 1$, and (c)
  $\rpdbinfinite\geq 1$.

  From (a) and \eqref{eq:rpdbstate-eff} from  $\Hpdb'$, we get (d)
  $v'\geq 1$ for all $v\in \pdbstate$ and (e) $\overline{v'}\geq 1$ for all
  $v\in P\setminus \pdbstate$.

  From (b) and \eqref{eq:actionvar}, we get (f) $v'\geq 1$ for all
  $v\in\add(a)$, (g) $\overline{v'}\geq 1$ for all
  $v\in\del(a)$, (h) $\eqvar{v}\geq 1$ for all
  $v\in\variables\setminus\effvars(a)$ and (i) $v\geq 1$ for all $v\in\pre(a)$.

  Since $a$ is inconsistent with $\pdbstate$, (*) there is a $v\in P\cap
  \add(a)$ with $\overline{v'}\geq 1$, or (**) there is a $v\in P\cap \del(a)$
  with $v'\geq  1$, or (***) there is a $v\in(\pre(a)\cap P)\setminus
  \effvars(a)$ with $\overline{v'}\geq 1$.  In cases (*) and (**), we can
  derive a contradiction with (f) and (e), or (g) and (e), respectively. In case (***), we can
  for this $v$ derive from (h) and \eqref{eq:eqvars} that $\overline{v}\geq 1$,
  contradicting (i).
\end{proof}

The next lemma allows us to derive the same constraint as the previous lemma but for
a consistent action.
\begin{lemma}
  For each action $a\in \actions$ and abstract state $\pdbstate$ with
  $d(\pdbstate)<\infty$ such that $a$ is consistent with $\pdbstate$, it is possible to derive
  $({\rpdbinfinite\land r_a)\rightarrow
  \overline{\rpdbstate{\pdbstate}}'}$ from $\ctrans\cup \Hpdb\cup \Hpdb'$.
  \label{lemma:PDBinftyconsistentaction}
\end{lemma}

\begin{proof}
Define $L = (\pre(a)\cap P)\cup(\pdbstate\setminus\effvars(a))\cup
\{\overline{v}\mid v\in (P\setminus\pdbstate)\setminus\effvars(a)\}$
and the corresponding reification
(a) $r_L \Leftrightarrow \sum_{\ell \in L} \ell \geq 1$.
Define $L^+ = \{\ell \in L \mid \ell=v \text{ for some }v \in P\}$ and
$L^- = L\setminus L^+$.

Define set $Q = \{\widehat{\pdbstate}\subseteq P \mid
L^+ \subseteq \widehat{\pdbstate} \text{ and }
L^- \subseteq P\setminus\widehat{\pdbstate}\}$.
This set contains exactly the abstract states $\widehat{\pdbstate}$ such that
$\app{\widehat{\pdbstate}}{\aalpha} = \pdbstate$.
Since $\pdbstate$ has a finite abstract goal distance, this must also be true
for all these $\widehat{\pdbstate}$.
We can use lemma \ref{lemma:partialstatesetextension} to derive
(b) $\overline{r_L} + \sum_{\rpdbstate{\widehat{\pdbstate}}\in Q}\rpdbstate{\widehat{\pdbstate}}\geq 1$.

We are now ready to establish the desired constraint with RUP:
Assume
(c) ${\rpdbstate{\pdbstate}}'\geq 1$,
(d) $r_a\geq 1$, and
(e) $\rpdbinfinite\geq 1$.

From (c) and \eqref{eq:rpdbstate-eff} from  $\Hpdb'$, we get
(f) $v'\geq 1$ for all $v\in \pdbstate$ and
(g) $\overline{v'}\geq 1$ for all $v\in P\setminus \pdbstate$.

From (d) and \eqref{eq:actionvar}, we get
(h) $v'\geq 1$ for all $v\in\add(a)$,
(i) $\overline{v'}\geq 1$ for all $v\in\del(a)$,
(j) $\eqvar{v}\geq 1$ for all $v\in\variables\setminus\effvars(a)$ and
(k) $v\geq 1$ for all $v\in\pre(a)$.

We derive from \eqref{eq:eqvars}, (f) and (j) that
(l) $v\geq 1$ for all $v\in \pdbstate\setminus\effvars(a)$
and analogously from (g) and (j) that
(m) $\overline{v}\geq 1$ for all $v\in (P\setminus\pdbstate)\setminus\effvars(a)$.

From (a) with (k), (l) and (m) we derive
(n) $r_L \geq 1$.

From (e) and \eqref{eq:rpdbinfinite-eff} we get that
(o) $\overline{\rpdbstate{\widehat{\pdbstate}}} \geq 1$ for all $\widehat{\pdbstate} \in S$ and in particular for all $\widehat{\pdbstate} \in Q$.

From (b) and (o) we get $\overline{r_L} \geq 1$, contradicting (n).
\end{proof}

The next lemma generalizes the previous lemmas from individual actions to the
entire transition relation:
\begin{lemma}
  For abstract state $\pdbstate$ with
  $d(\pdbstate)<\infty$, RUP can derive
  $\ctask\cup \Hpdb\cup \Hpdb' \vdash (\rpdbinfinite\land
  \rtrans)\rightarrow \overline{{\rpdbstate{\pdbstate}}'}$.
  \label{lemma:PDBinftytransition}
\end{lemma}
\begin{proof}
  For each action $a\in \actions$, we establish (a)
  $(\rpdbinfinite\land r_a)\rightarrow \overline{{\rpdbstate{\pdbstate}}'}$
  by means of lemmas \ref{lemma:PDBinftyinconsistentaction} and
  \ref{lemma:PDBinftyconsistentaction}, respectively.
  
  Afterwards, the desired constraint can be derived by RUP, assuming (b)
  ${\rpdbstate{\pdbstate}}'\geq 1$, (c) $\rtrans\geq 1$ and (d)
  $\rpdbinfinite\geq 1$. For every $a\in\actions$, we get from (b), (d) and (a)
  that $\overline{r_a}\geq 1$. With \eqref{eq:rtrans}, this yields
  $\overline{\rtrans}\geq 1$, contradicting (c).
\end{proof}

The following lemma allows us to derive that any successor of any state in a region
that is not backwards-reachable, is contained in the invariant:
\begin{lemma}
  It is possible to derive 
$(\rpdbinfinite\land \rtrans)
  \rightarrow r'_{\textup{PDB}}$ from $\ctrans\cup \Hpdb\cup \Hpdb'$.
  \label{lemma:PDBinftytransition2}
\end{lemma}

\begin{proof}
We first establish
(a) $(\rpdbinfinite \land \rtrans)\rightarrow \overline{{\rpdbstate{\pdbstate}}'}$
for all abstract states $\pdbstate\in S$ using lemma \ref{lemma:PDBinftytransition}.

Then we can derive the desired constraint by RUP, assuming
(b) $\rpdbinfinite \geq 1$,
(c) $\rtrans \geq 1$, and
(d) $\overline{r'_{\textup{PDB}}} \geq 1$.
From (b) and (c) and (a) we obtain
(e) $\overline{{\rpdbstate{\pdbstate}}'} \geq 1$ for all abstract states $\pdbstate\in S$.
Using the primed version of \eqref{eq:rpdbinfinite-eff} from $\Hpdb'$ and (e), we obtain
(f) ${\rpdbinfinite}' \geq 1$, which together with the primed version of
\eqref{eq:rPDB-eff} from $\Hpdb'$ yields $r'_{\textup{PDB}} \geq 1$ in
contradiction to (d).
\end{proof}

The final step for the inductivity lemma generalizes the previous results to the
entire PB circuit of the heuristic.
\begin{lemma}
  It is possible to derive the inductivity lemma $(r_{\textup{PDB}} \land\rtrans)\rightarrow
  r'_{\textup{PDB}}$ from $\ctrans\cup \Hpdb\cup
  \Hpdb'$.
\end{lemma}

\begin{proof}
We first establish with lemma \ref{lemma:PDB3-eff} that
(a) $(\rpdbstategeg{\pdbstate}\land \rtrans) \rightarrow r'_{\textup{PDB}}$
for each abstract state $\pdbstate$ with $d(\pdbstate) < \infty$
and with lemma \ref{lemma:PDBinftytransition2} that
(b) $(\rpdbinfinite\land \rtrans)\rightarrow r'_{\textup{PDB}}$ holds.

The desired constraint can now be derived by RUP, assuming
(c) $r_{\textup{PDB}}\geq 1$,
(d) $\rtrans\geq 1$, and
(e) $\overline{r'_{\textup{PDB}}}\geq 1$.
Using (d) and (e) with (a), we can derive for each abstract state $\aalpha$
with $d(\aalpha)<\infty$ that
(f) $\overline{\rpdbstategeg{\pdbstate}}\geq 1$.
Using (d) and (e) with (b) we get
(g) $\overline{\rpdbinfinite}\geq 1$.
From \eqref{eq:rPDB-eff} with (f) and (g) we get $\overline{r_{\textup{PDB}}} \geq 1$,
contradicting (c).
\end{proof}

}

\end{document}